\documentclass[runningheads]{llncs}
\usepackage[T1]{fontenc}
\usepackage{graphicx}
\usepackage{booktabs}
\usepackage[misc]{ifsym}

\usepackage{xcolor}
\usepackage{subcaption}	
\usepackage[utf8]{inputenc}
\usepackage[english]{babel}
\usepackage{algpseudocode} 
\usepackage{algorithm}
\usepackage{csquotes}
\usepackage{float}
\usepackage{authblk}

\usepackage{amsmath, amssymb} 

\newcommand{\Cm}{\mathbf{C}}  
\newcommand{\Dm}{\mathbf{D}}  
\newcommand{\Po}{\mathbf{P}^*}  
\newcommand{\Ps}{\mathbf{P}^s}  
\newcommand{\Pa}{\mathbf{P}'}  
\newcommand{\Paf}{\mathbf{P}^{\scriptscriptstyle\sim}} 
\newcommand{\qa}{q'}  
\newcommand{\dtwo}{d^*}  
\newcommand{\dtwa}{d^{\scriptscriptstyle\sim}}  

\newcommand{\lo}{L}  
\newcommand{\co}{c^*}  
\newcommand{\ca}{c^{\scriptscriptstyle\sim}}  
\newcommand{\cai}{c_i^{\scriptscriptstyle\sim}}  
\newcommand{\coi}{c_i^*}  
\newcommand{\cp}{c^{\prime}}  
\newcommand{\loi}{l_i}

\DeclareMathOperator{\ldist}{cost}  
\newcommand{\ldistm}{\phi}
\newcommand{\ldisti}{\phi^{-1}}
\newcommand{\tss}{\mathbf{s}}
\newcommand{\tsf}{\mathbf{s_{1}}}  
\newcommand{\tst}{\mathbf{s_{2}}}  
\newcommand{\ubm}{\delta_{\rm rel}}  
\newcommand{\uba}{\delta_{\rm abs}}  
\newcommand{\gtm}{\Delta_{\rm rel}}  
\newcommand{\gta}{\Delta_{\rm abs}}  

\algblockdefx{ARGUMENTS}{ENDARGUMENTS}%
[1]{\textbf{Arguments:} #1}%
[0]{\addtocounter{ALG@line}{-1}\vspace{-\baselineskip}}
\algblockdefx{RETURNS}{ENDRETURNS}%
[1]{\textbf{Returns:} #1}%
[0]{\addtocounter{ALG@line}{-1}\vspace{-\baselineskip}}

\begin{document}
\title{Warping and Matching Subsequences Between Time Series}
\author{Simiao Lin\inst{1,2} \and Wannes Meert\inst{1,2}  \and Pieter Robberechts\inst{1,2}  \and Hendrik Blockeel\inst{1,2}}
\institute{Dept. of Computer Science, KU Leuven, Leuven, Belgium  \and Leuven.AI, Leuven, Belgium}
\date{June 2025}

\maketitle


\begin{abstract}
Comparing time series is essential in various tasks such as clustering and classification. While elastic distance measures that allow warping provide a robust quantitative comparison, a qualitative comparison on top of them is missing. Traditional visualizations focus on point-to-point alignment and do not convey the broader structural relationships at the level of subsequences. This limitation makes it difficult to understand how and where one time series shifts, speeds up or slows down 
with respect to another. To address this, we propose a novel technique that simplifies the warping path to highlight, quantify and visualize key transformations (shift,  compression, difference in amplitude). By offering a clearer representation of how subsequences match between time series, our method enhances interpretability in time series comparison.

\end{abstract}

\section{Introduction}

Comparing two time series is a core operation in many tasks 
such as clustering, classification, segmentation, motif discovery or anomaly detection. Typically, such a comparison is performed quantitatively and a \emph{distance} is returned that expresses dissimilarity between the series. 
Examples include 
Euclidean distance, Dynamic Time Warping (DTW) distance \cite{Sakoe1978TASS}, or Move-Split-Merge (MSM) distance \cite{Stefan2012TKDE}. The latter two are \emph{elastic distances}: they account for warping that might happen between two similar occurrences, such as movements that are executed with slightly different speeds, or heartbeats that vary according to heart rate. 

Elastic distances are popular because they are simple, allow for a robust comparison between two time series, and often have competitive performance in classification~\cite{Bagnall2017DMKD}. 
A challenge with elastic distances, however, is performing a qualitative comparison.
While it is clear how individual time points in one series correspond to those in another, understanding how subsequences of time points in one series map to subsequences in the other remains difficult.
Typical visualizations map a time point to its warped time point in the other series, optionally 
using colors to indicate where one time series speeds up or slows down with respect to the other time series (see Figure~\ref{fig:intro}) \cite{Budikova2022ICSSA,mcfee2015librosa,meert_2020_7158824,Urribarri2020CACIC}. 
Such visualizations are especially popular in domains such as human movement \cite{Urribarri2020CACIC} or sound processing \cite{mcfee2015librosa} where understanding, via visualization, how two series differ is important to understand the data. 

In such graphs, it is challenging to identify at a glance which specific subsequence in one time series corresponds to which subsequence in the other, where these subsequences start and end, and how they relate to each other. For example, a subsequence in one time series might be a compressed version of a subsequence in another. This can be observed in Figure~\ref{fig:intro} where both time series exhibit a sine wave toward the end. Compared to the top time series, the sine wave in the bottom time series is compressed. However, this information is not clearly conveyed.
Coloring the individual warping lines improves the visualization, but it is still hard to see where behavior is consistent over a subsequence.

In this work, we present a technique that identifies and quantifies the relationship between subsequences with similar shapes in two time series.\footnote{Python package available at \url{https://github.com/wannesm/dtaidistance} and \url{https://pypi.org/project/dtaidistance}}
We utilize DTW and its corresponding optimal warping path to measure the distance, and further simplify this optimal warping path. This enables a new type of visualization (see Figure~\ref{fig:intro_segmented}) in which one can easily observe how subsequences of one series need to be (1) shifted (forward or backward) 
and (2) compressed (for brevity, we use the term ``compression'' to mean both compression and expansion in this text), to best fit the other series. The absolute amount of compression of a subsequence is visualized using an orange block (see Figure~\ref{fig:intro_segmented}, ``compression''), and
amplitude differences after warping are visually represented by an orange-shaded region (``amplitude difference'' in Figure~\ref{fig:intro_segmented}).

\begin{figure}[t]
\begin{subfigure}[t]{.45\textwidth}
	(a)
    \raisebox{-0.8\height}{\includegraphics[width=\linewidth]{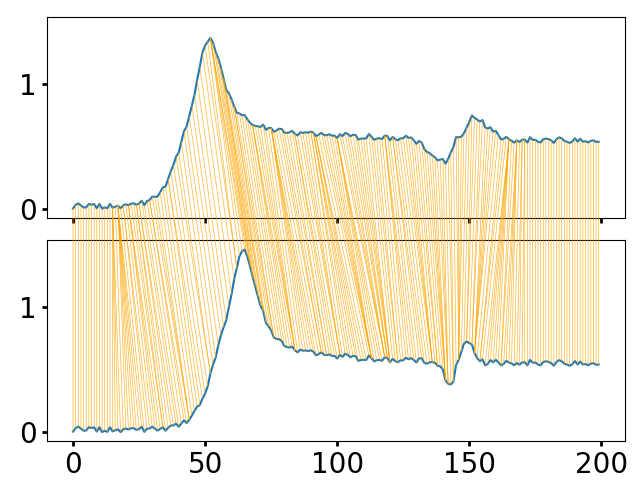}}

\end{subfigure}%
\hfill%
\begin{subfigure}[t]{.50\textwidth}
    (b)
	\raisebox{-0.8\height}{\includegraphics[width=\linewidth]{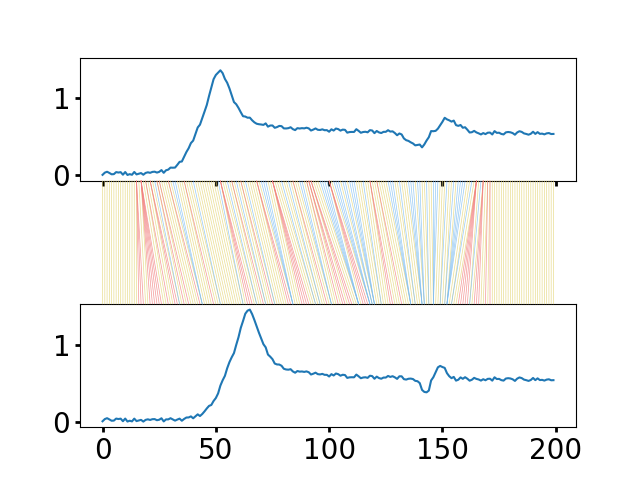}}
\end{subfigure}%
\caption{Existing visualizations for DTW. (a) Connecting matching values at all time points \cite{meert_2020_7158824}. (b) Connecting the indices and coloring the misalignment \cite{Urribarri2020CACIC}: red/blue means one series is sped up / slowed down w.r.t. the other.}
\label{fig:intro}
\end{figure}

\begin{figure}
\centering
\includegraphics[width=0.7\linewidth]{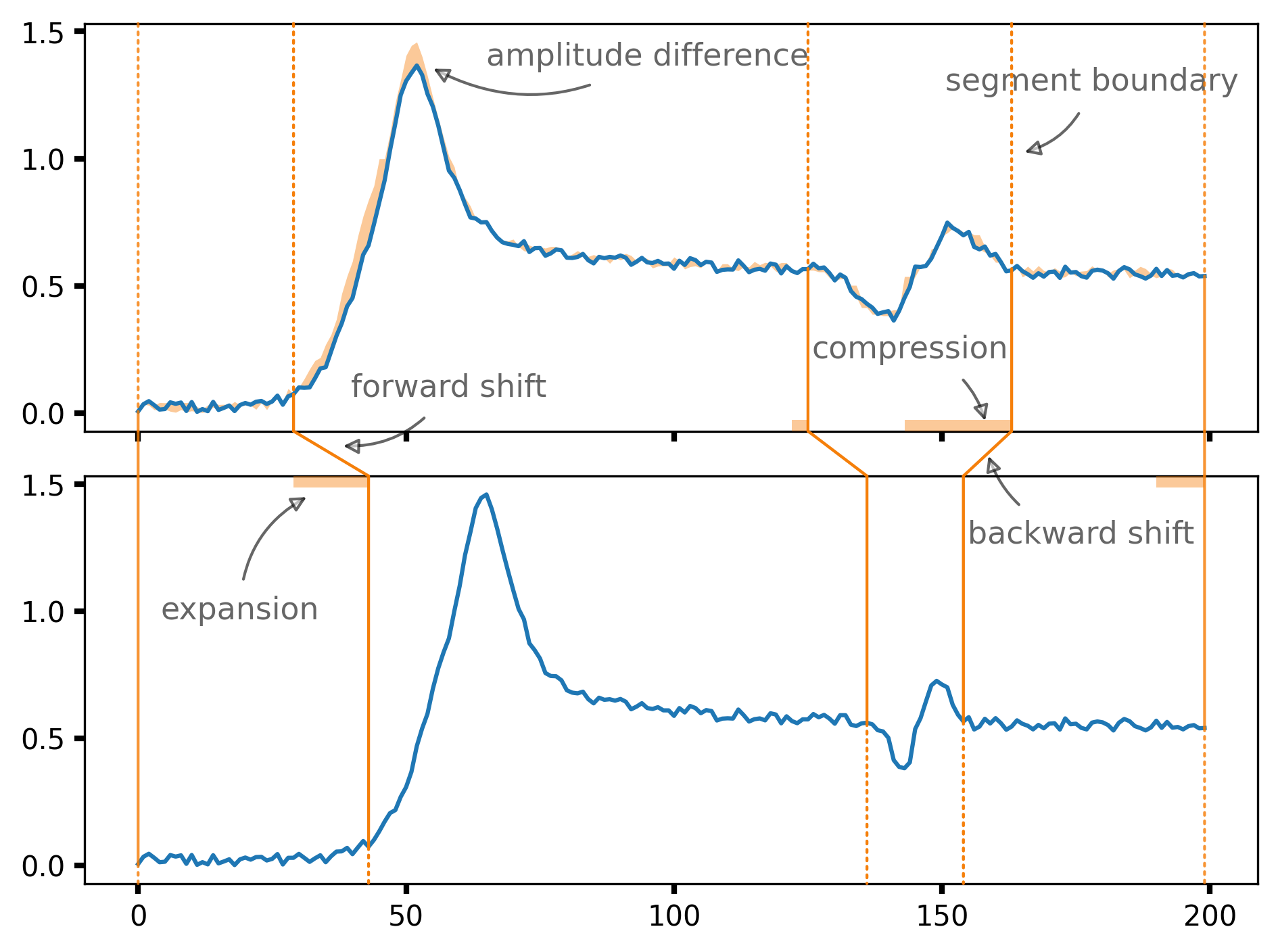}
\caption{Our new type of visualization shows the segments where the two time series are similar up to a constant time-shift and compression. 
}

\label{fig:intro_segmented}
\end{figure}


\section{Background}
\subsection{Time Series}
A \emph{time series} is a sequence of values measured over time, where each value is associated with a specific point in time. We assume the time points are spaced at equidistant time intervals and indexed consecutively: $\tss = \{v_1, v_2, \ldots, v_n\}$, with $|\mathbf{s}|=n$ the length.
When referring to the $i^{\textrm{th}}$ item in the series we use the syntax $\tss(i)=v_i$. A range of items from the $i^{\textrm{th}}$ to the $j^{\textrm{th}}$ item is $\tss(i:j)$. A list of indices can be used to generate a new sequence from the original sequence, thus $\tss(\{1,3,3\})=\{v_1,v_3,v_3\}$.
When $v_i$ is a scalar, it is a univariate time series. If $v_i$ is a tuple with a dimension larger than one, it is a multivariate time series.

\subsection{Warping Path}

A \emph{warping path} is a sequence of index pairs that represents a many-to-many matching between indices in one time series to indices of another time series. For example, a path $\mathbf{P} = \{(1,1), (1,2), (2,3), (3,3)\}$ between $\mathbf{s}_1$ and $\mathbf{s}_2$ states that value $\mathbf{s}_1(1)$ is matched to $\mathbf{s}_2(1)$ and $\mathbf{s}_2(2)$, $\mathbf{s}_1(2)$ is matched to $\mathbf{s}_2(3)$, etc. A warping path must satisfy three conditions \cite{Sakoe1978TASS}: (1) boundaries: a path starts with $(1,1)$ and ends with $(|\mathbf{s}_1|,|\mathbf{s}_2|)$; (2) continuity: every index occurs at least once for both series; (3) monotonicity: the indices from one series can only remain the same or increase.
The notation earlier introduced for sequences is also used for warping paths: $\mathbf{P}(i)$ is the $i^{\textrm{th}}$ item in sequence $\mathbf{P}$, etc.

\subsection{Cost and Distance According to a Warping Path}


A \emph{cost function} $\ldist(v_i,v_j)$ expresses how similar two values $v_i$ and $v_j$ in a time series are.\footnote{The cost function is also referred to as the local distance (function) in literature.} Cost functions are usually of the form $\ldist(x,y) = \ldistm(||x-y||)$ 
with $\ldistm$ a monotonically increasing function, such as $\ldistm(z)=z^\lambda$; that is, they depend monotonically on a norm of the difference, like the $p$-norm. Common instances are the Manhattan distance ($p=1, \lambda=1$) and squared Euclidean distance ($p=2, \lambda=2$). For univariate time series, this boils down to absolute and squared difference. But any real value can be chosen for $\lambda$ \cite{Herrmann2023DMKD}.

The \emph{cost of a warping path} is the sum of costs of all the matched value pairs: $c = \sum_{(n,m) \in \mathbf{P}}\ldist(\tsf(n),\tst(m))$.
The \emph{distance according to the path} is $d = \ldisti(c)$ 
(e.g., $d = c^{1/\lambda}$). Because the algorithms discussed below work in ``cost space'', while distance has a more natural interpretation in the original data space, we will occasionally use $\ldistm$ and $\ldisti$ to switch between these spaces.

\subsection{Dynamic Time Warping}

\emph{Dynamic Time Warping (DTW)} is a method to find a warping path that minimizes the distance according to that path. This is the warping path where the matched values are as close as possible to each other. Therefore, this \emph{DTW distance} expresses the dissimilarity (in shape) between two time series \cite{Sakoe1978TASS}. 

To find the \emph{optimal warping path} $\Po$ and the associated \emph{optimal distance} $\dtwo$, dynamic programming is used. First, the cost matrix is computed $\Cm(n,m) = \ldist(\tsf(n),\tst(m))$. Then, the accumulated cost matrix $\Dm$ is filled in:
\begin{eqnarray*}
        \Dm(1,1) &=& \Cm(1, 1) \\
        \Dm(1,m) &=& \Cm(1,m) + \Dm(1, m-1) \quad \text{if } m > 1\\ 
	\Dm(n,1) &=& \Cm(n,1) + \Dm(n-1, 1) \quad \text{if } n > 1\\
	\Dm(n,m) &=& \Cm(n,m) + {\min}_{(f, t) \in \{(1,1), (1, 0), (0,1)\}}\Dm(n-f,m-t)\\
     &&\qquad \text{if } m>1 ~\text{and}~ n >1
\end{eqnarray*}%

Next, the optimal point-to-point alignment is found between $\tsf$ and $\tst$ using backtracking:
\begin{eqnarray*}
	q_{\lo} &=& (|\tsf|, |\tst|) \\
    q_{l-1} &=& q_l - {\arg\min}_{(f, t) \in \{(1,1), (1, 0), (0,1)\}}\ \Dm(q_l-(f,t))   \\
    q_1 &=& (1, 1) 
\end{eqnarray*}
\noindent Finally, the optimal warping path $\Po = (q_1, q_2, \ldots, q_{\lo})$ is constructed, where $\lo$ is the length of $\Po$. 
The DTW distance is $\dtwo = \ldisti(\co) = \ldisti(\sum_{q_i\in \Po} \Cm(q_i)) $, where $\co$ is the total cost along $\Po$.

A warping path $\mathbf{P}$ can be visualized in the cost matrix by highlighting for each $(i,j) \in \mathbf{P}$ the cell on row $i$ and column $j$ (see Fig.~\ref{fig:intro_path}).
The method we propose in this paper can be thought of as simplifying a warping path by approximating it with linear segments.


\begin{figure}[t]
\begin{subfigure}[t]{.45\textwidth}
    \includegraphics[width=\linewidth]{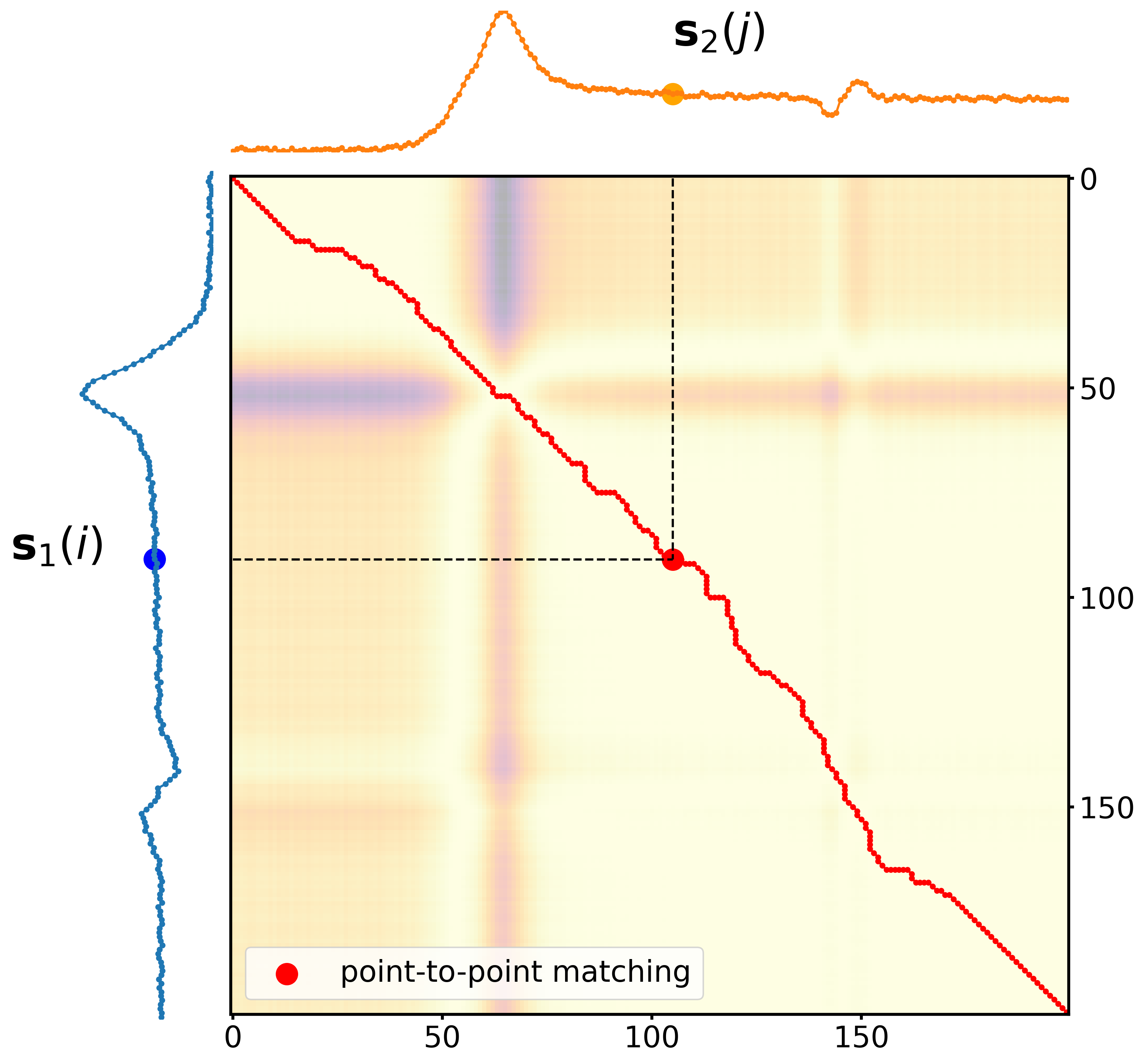}
\end{subfigure}
\hfill
\begin{subfigure}[t]{.45\textwidth}             
    \includegraphics[width=\linewidth]{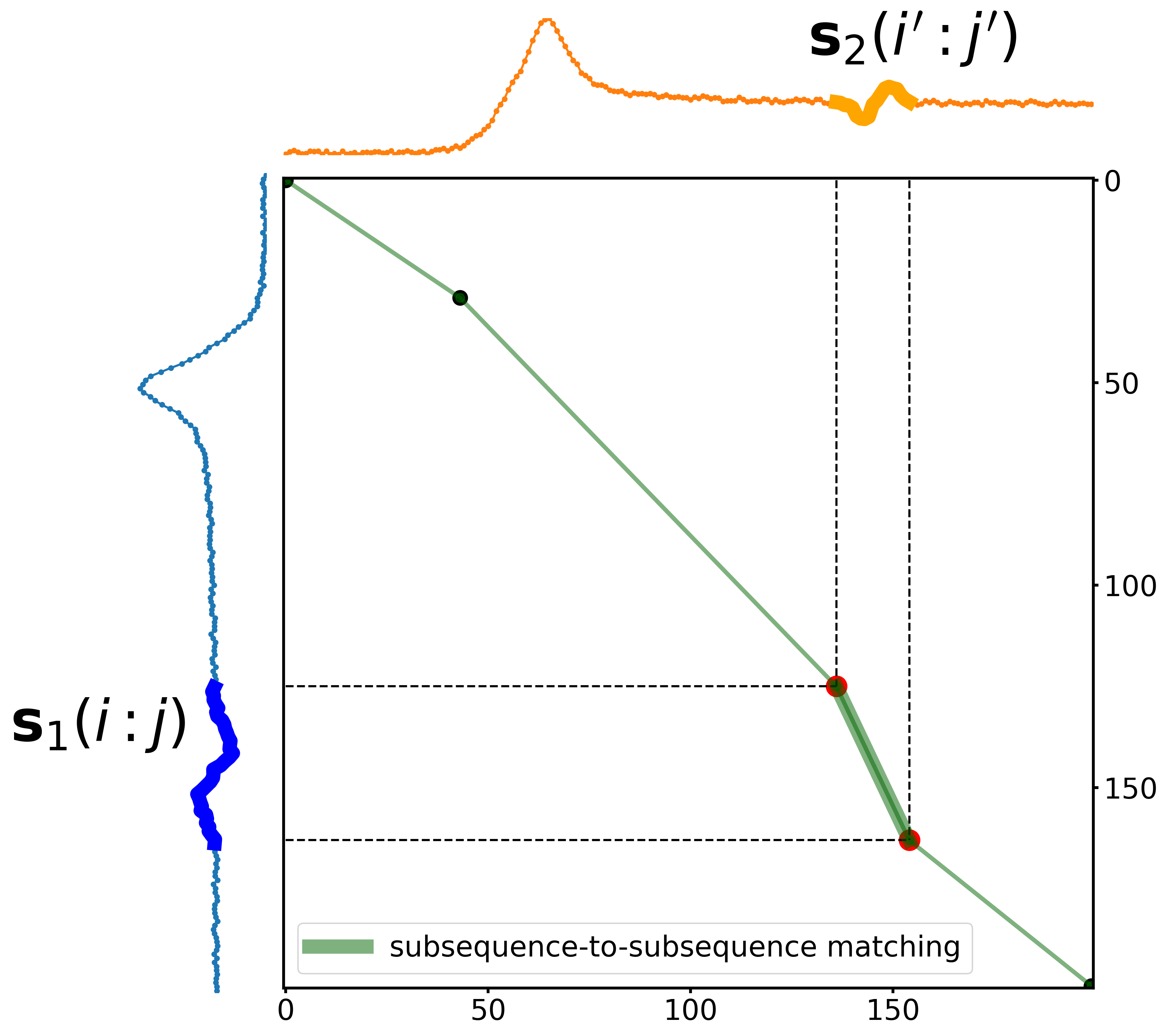}
\end{subfigure}
\caption{Left: A cost matrix constructed for DTW (lighter colors mean smaller values), and the optimal warping path computed from it (red line). Right: A simplified warping path as constructed by our method.}
\label{fig:intro_path}

\end{figure}


\section{Problem Statement}

The goal of this work is to explain comparisons between time series by identifying a higher-level matching between time series: one that maps subsequences to similar subsequences, rather than mapping individual time points.

A key element of the proposed method is the concept of a uniform subsequence mapping.  Intuitively, a warping path $\mathbf{P}$ maps subsequence $\mathbf{s}_1(i:j)$ {\em uniformly} to subsequence $\mathbf{s}_2(i':j')$ 
if the image of each index can be found by linear interpolation; that is, if $a = i + \eta \cdot (j-i)$, then index $a$ is mapped to $a' = i' + \eta \cdot (j'-i')$, rounded to the nearest integer.\footnote{The precise relationship is slightly more complex because the mapping is many-to-many: if the slope of the line is above 1, the roles of $a$ and $a'$ are reversed.  A standard algorithm such as Bresenham's \cite{Bresenham1987CGA} can be used to calculate this mapping.}
In the visualization of the warping path in a cost matrix, this means a straight line connects the points $(i,i')$ and $(j,j')$ (see Figure~\ref{fig:intro_path} for an example). 

Our goal now is to find a warping path that consists of a relatively small number of uniform subsequence mappings, and yields a distance close to the optimal distance (the user can determine how close). The result is a segmentation of the path and thus also the time series. We use the term ``segment'' both for path segments and for the corresponding segments in the time series. 
The segmentation shown in Figure~\ref{fig:intro_segmented} shows four such segments.

\section{Dynamic Subsequence Warping}
To address the above problem, we propose a novel method called Dynamic Subsequence Warping (DSW).
The method works in two phases. In a first phase, it constructs increasingly complex subsequence-based approximations of the optimal path. Because that approach is greedy, it may result in more segments than needed. Therefore, in a second phase, a pass is made to further simplify the solution. Sections \ref{sec:phase1} and \ref{sec:ub} focus on the first phase, Section \ref{sec:phase2} on the second.

\subsection{Phase 1: Identifying uniform subsequence mappings}
\label{sec:phase1}
The first phase of our algorithm is inspired by the Ramer-Douglas-Peucker (RDP) algorithm~\cite{Ramer1972CGIP,Douglas1973Carto}.  RDP simplifies 
a curve by reducing the number of points in a path while retaining its essential shape, and is widely used in visualization due to its efficiency and effectiveness. We adopt the top-down simplification strategy in RDP, starting with an attempt to simplify the entire optimal warping path with a straight line connecting its endpoints.

A key step in the simplification framework 
is assessing whether a given path (the entire path or a partial path) can be simplified with a straight line connecting its two endpoints, based on a tolerance criterion. If the criterion is met, the simplification is accepted; otherwise, the path is split at the point farthest from the proposed straight line into two partial paths 
(see Figure~\ref{fig:alg_example} for an illustration).
This process continues recursively until no further splitting is needed. 

The main difference between the proposed method and RDP is that in RDP, the criterion used for determining {\em where} to split (visual distance to the linear approximation) is also used to determine {\em whether} to split, whereas our method decouples these: {\em whether} to split is based on the {\em cost} of the paths.

\begin{figure}[t]
    \begin{subfigure}[t]{.48\textwidth}
    \includegraphics[width=\linewidth]{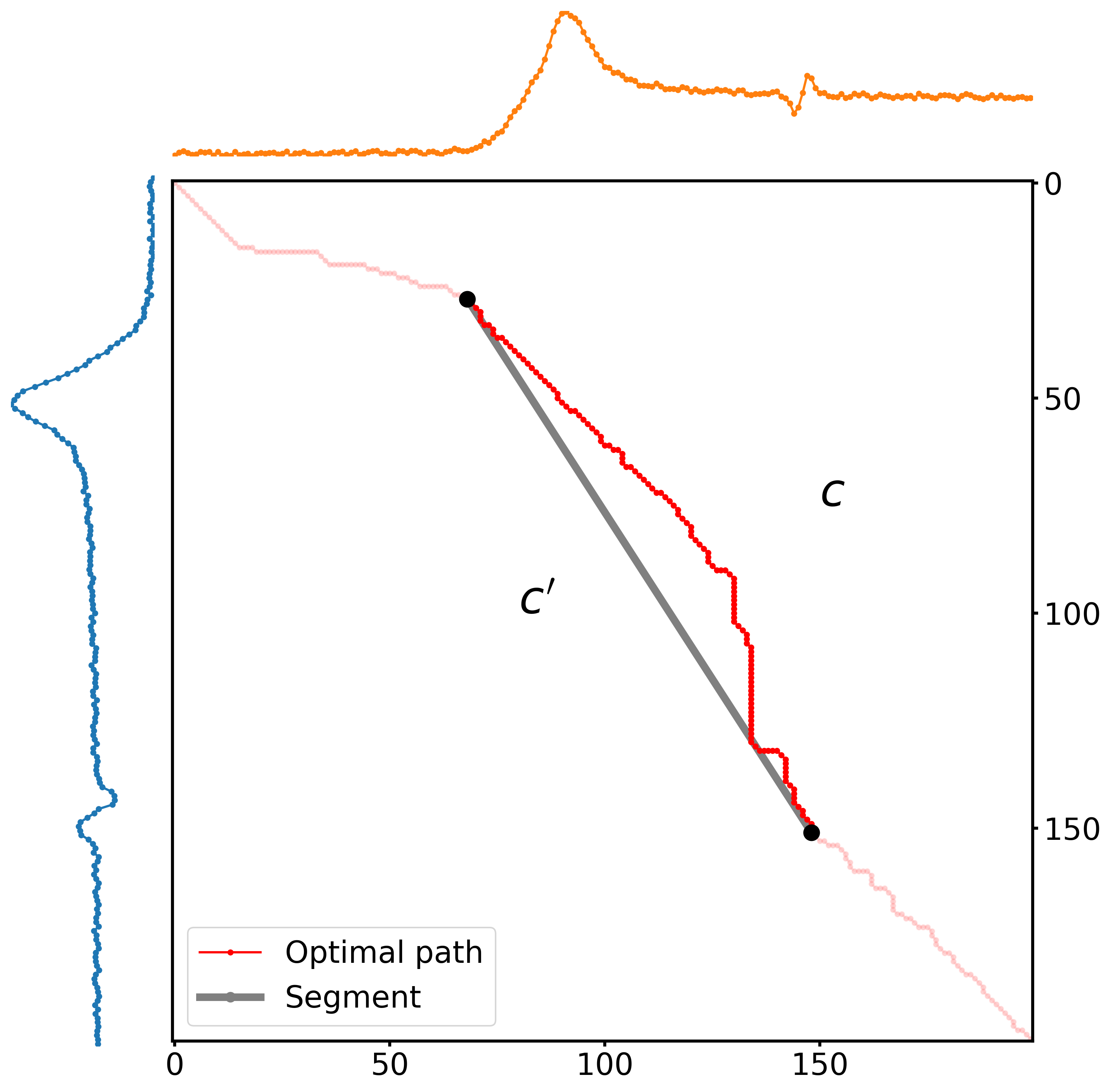}
    \caption{A segment of the optimal path (red) is considered for replacement by a linear segment (grey).  If the cost of the linear segment ($c'$ in the algorithm) does not exceed the optimal cost ($c$ in the algorithm) by more than some tolerance, it is accepted.} 
    \end{subfigure}
    \hfill
    \begin{subfigure}[t]{.48\textwidth}
    \includegraphics[width=\linewidth]{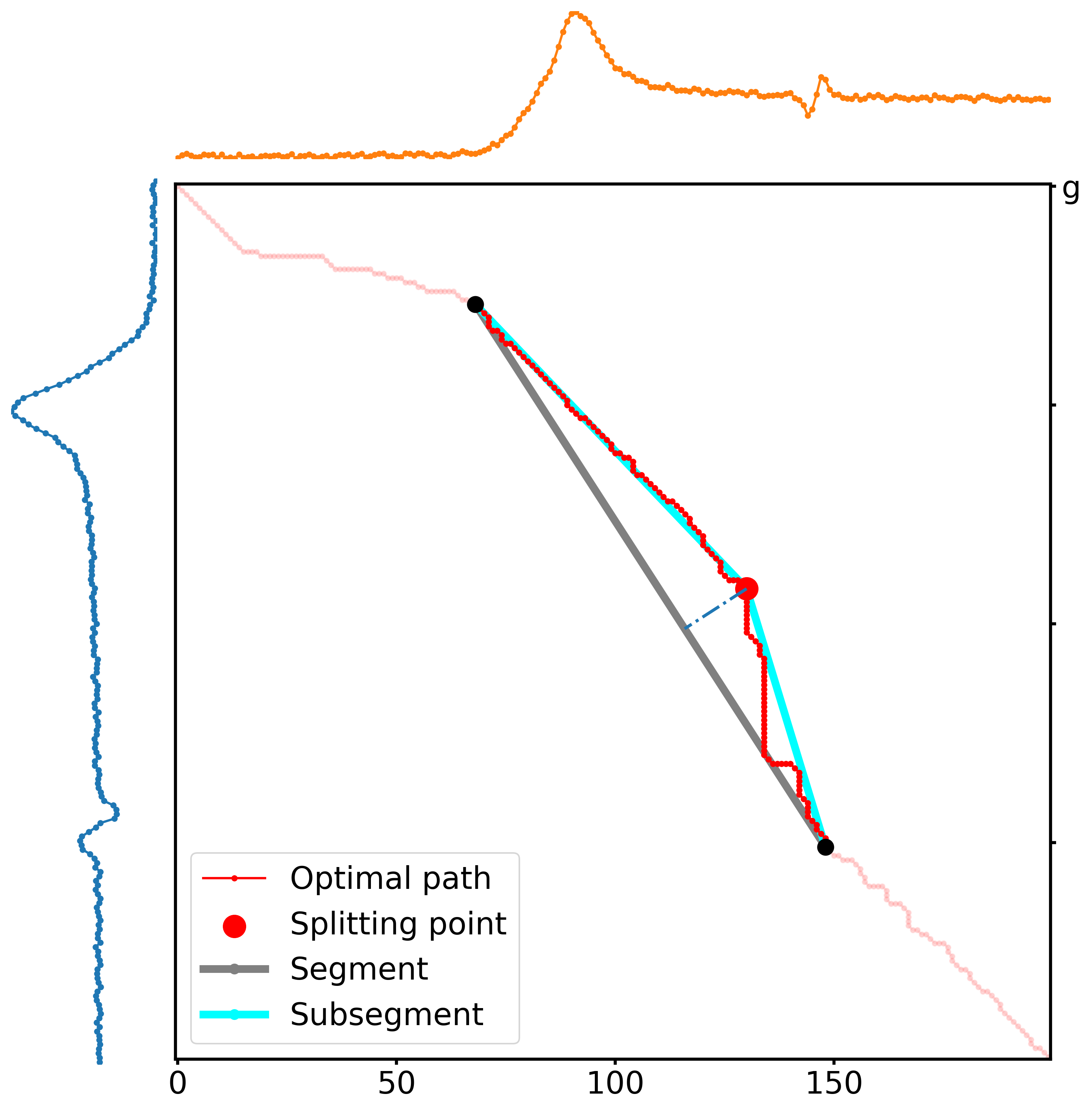}
    \caption{If $c'$ does exceed $c$ beyond the tolerance, the path is split at the point farthest from the linear approximation, and linear approximations are constructed for the subpaths. 
    }
    \end{subfigure}
\caption{Splitting a path segment.}
\label{fig:alg_example}
\end{figure}
Algorithm~\ref{alg:sop} describes the method in detail. The while loop repeatedly takes a path segment $\Po(b:e)$ (starting with the full path), and compares the cost $c$ of the corresponding optimal path with the cost $\cp$ of the linear approximation.  The cost of the optimal path segment is obtained from the cumulative cost matrix as $c = \Dm(\Po(e))-\Dm(\Po(b)) = \sum_{i=b+1}^{e} \Cm(\Po(i))$. The cost of the linear path segment is obtained by interpolating $\{\Po(b), \Po(e)\}$ into a continuous path $\Pa = (\qa_1, \qa_2, \ldots ,\qa_k)$ (excluding $\Po(b)$),
using rasterization under Bresenham's line algorithm~\cite{Bresenham1987CGA} 
and calculating $\cp = \sum_{i=1}^k \Cm(\qa_i)$ (details in  Appendix~\ref{apx:bresenham}).
The cost of boundary point $\Po(b)$ is excluded in both calculations, to avoid that when adding the costs of all segments, boundary points are added twice. 
By definition, $\cp \geq c$, as $c$ is the cost of the optimal path.  The approximation is kept if $\cp$ is below a certain threshold (line 15 in the algorithm). 
This threshold differs from the original one in RDP. We motivate it in the next subsection. 

The result of this procedure is a filtered set of points $\Ps \subset \Po$ where only the key points from $\Po$ are retained.
$\Ps$ represents a piecewise linear path, with each pair of subsequent items defining a path segment.


\begin{algorithm}[t]
\caption{Phase 1: Identifying uniform subsequence mappings.}
\label{alg:sop}
\begin{algorithmic}[1]
\Procedure{SegmentWarpingPath}{$\Po,\Cm,\uba, \ubm$}
\ARGUMENTS{}
\State{$\Po$: Optimal warping path}
\State{$\Cm$: Cost matrix}
\State {$\uba$, $\ubm$: Absolute and relative tolerance parameter}
\ENDARGUMENTS

\State {$L \gets |\Po|$}
\State $\mathbf{Q} \gets \{ (1,L)\}$  \Comment{List of path segments to be considered} 
\State $\mathbf{R} \gets \{\}$ \Comment{Indices of key points found till now} 
\While{$\mathbf{Q} \not= \emptyset$}
\State {Take a path segment $(b, e)$ from $\mathbf{Q}$. }
\State $c^\prime \gets \Call{LinearPathCost}{\Po, \Cm, b, e}$\Comment{Cost of the simplified path}
\State $c \gets \Call{OptimalPathCost}{\Po, \Cm, b, e}$ \Comment{Cost of the original path}
\State $l \gets e - b$ 
\If{$c^\prime \leq \max(c+\frac{l}{L}\cdot \uba, c\cdot (1+\ubm))$}\Comment{Tolerance criterion}
\State {Add $b$ and $e$ to $\mathbf{R}$.}  \Comment{Keep segment}
\Else
\State $s \gets \Call{ArgmaxSpatialDistFromLine}{\Po, b, e }$ \Comment{Split further}
\State {Add $(b, s)$ and $(s, e)$ to $\mathbf{Q}$.}
\EndIf
\EndWhile\label{alg:sop-while}
\State $\Ps \gets \{\Po(i) \mid i \in \mathbf{R}\}$ 
\State $\textbf{return } \Ps$
\EndProcedure
\end{algorithmic}
\end{algorithm}

\subsection{Tolerance Criterion}
\label{sec:ub}

The tolerance criterion used for each individual segment (line 15 in Algorithm~\ref{alg:sop}) is a crucial component of our simplification framework and requires some discussion. The criterion states that the cost $\cp$ of a linear segment must not exceed the cost $c$ of the corresponding optimal path by a given margin:
\[
\cp \leq \max\bigl(c+\frac{l}{L}\cdot \uba, c\cdot (1+\ubm)\bigr)
\]
It is determined by two tolerance parameters of the algorithm: an absolute tolerance $\delta_{\rm abs}$ and a relative tolerance $\delta_{\rm rel}$.  Both are interpreted as the allowed increase of cost for the full sequence matching.  When considering one particular segment, $\delta_{\rm rel}$ can still be used as is (if all segments have a cost that is at most, say, 5\% higher than the optimal cost for that segment, then the same will hold for the cost of the total path), but $\delta_{\rm abs}$ must be multiplied by $l/L$ with $L$ the length of the path and $l$ the length of the path segment under consideration. 


Below, we show how setting the $\uba$ and $\ubm$ hyperparameters leads to different types of guarantees on how close the distance of the simplified matching is to the DTW distance. For conciseness, we prove them in a different order than they are stated: we will prove the third and then show how the first two follow from that. We use $\dtwa$ to denote the distance according to the simplified path.

\paragraph{Absolute tolerance criterion.} Setting  $\uba = \ldistm(\dtwo+\gta) - \ldistm(\dtwo)$ and $\ubm=0$ guarantees the following bound on the distance of the entire path:
\[
    \dtwa \leq \dtwo  + \gta
\]
    
\paragraph{Relative tolerance criterion.}  Setting $\ubm =  \ldistm( \dtwo\cdot\gtm)  / \ldistm(\dtwo)$ and $\uba=0$ guarantees that the distance will not exceed the DTW distance by more than a certain proportion $\gtm$ of the DTW distance:
\[
    \dtwa \leq \dtwo \cdot (1+ \gtm)
\]

    

\paragraph{Combined relative and absolute tolerance criterion.}
While the previous two bounds are easier to understand, in practice we often need the combination of both. One reason is that using only a relative tolerance $\ubm$ precludes simplification when a segment's optimal path has a cost near zero (e.g., the path overfits noise); in such cases an absolute tolerance on top of the relative one is useful. The combination of both criteria is slightly more complicated but for the special case where  
$\ldistm$ is convex and monotonically increasing (and hence, $\ldisti$ concave and monotonically increasing), which the often used $\ldistm(z) = z^\lambda$ satisfies for $\lambda \geq 1$, a bound of the form
\[
    \dtwa \leq \dtwo \cdot (1+\gtm) + \gta
\]
can be proven, as the following proposition shows.

\begin{proposition}
For any convex and monotonically increasing $\ldistm$,
setting $\uba = \ldistm(\dtwo+\gta) - \ldistm(\dtwo) $ and  $\ubm = \ldistm( \dtwo\cdot\gtm)  / \ldistm(\dtwo)$ 
guarantees that \[
    \dtwa \leq \dtwo \cdot (1+\gtm) + \gta
\]
\end{proposition}

\begin{proof}
\label{proof:combi_global_ub}
 Let the total cost and the length of the optimal warping path $\Po$  be $\co$ and $\lo$ respectively. Assume $\Po$ is simplified to $n$ segments. Let the cost of the $i$-th rasterized segment be $\cai$. Let the cost, and the length of the optimal path on the $i$-th segment be $\coi$, and $\loi$ respectively. Let the accumulated cost on the whole simplified path be $\ca$. Note that (a) $ \sum_{i=1}^n\cai + \Cm(0,0) = \ca$, (b) $\sum_{i=1}^n\coi + \Cm(0,0) = \co$, and (c) $\sum_{i=1}^n\loi + 1 = \lo$.
%

As the maximum of two positive numbers cannot exceed their sum, we have
\begin{align*}
    \cai &\leq \max (\coi \cdot (1+\ubm), \coi + \frac{\loi}{\lo}  \cdot \uba) \\
         & = \coi + \max (\coi \cdot \ubm, \frac{\loi}{\lo}  \cdot \uba) \\
         &\leq \coi + \coi \cdot \ubm + \frac{\loi}{\lo}  \cdot \uba 
\end{align*}
Summing over all segments and adding $\Cm(0,0)$ gives
\begin{align*}
    \sum_{i=1}^n \cai + \Cm(0,0) &\leq \sum_{i=1}^n \coi + \ubm \sum_{i=1}^n \coi + \sum_{i=1}^n \frac{\loi}{\lo}  \cdot \uba + \Cm(0,0)
\end{align*}
Applying (a) to the left hand side and (b) and (c) to the right hand side yields, after simplification and reordering terms,
\begin{align*}
    \ca &\leq \ubm \cdot \co + \co + \uba
\end{align*}
Moving to the distance domain and applying the definition of $\uba$ and $\ubm$ in terms of $\gtm$ and $\gta$, we get
\begin{align*}
    \ldistm(\dtwa) & \leq \ubm \cdot \ldistm(\dtwo) + \ldistm(\dtwo) + \uba \\
    & = \ldistm(\dtwo \cdot \gtm) + \ldistm (\dtwo + \gta)
\end{align*}
With $\ldisti$ monotonically increasing (hence, order-preserving) and concave (hence, $\ldisti(x+y) \leq \ldisti(x)+\ldisti(y)$ for $x,y \geq 0$), this implies
\begin{align*}
             \ldisti(\ldistm(\dtwa)) & \leq \ldisti (\ldistm(\dtwo \cdot \gtm) + \ldistm(\dtwo + \gta))\\
             & \leq \ldisti (\ldistm(\dtwo \cdot \gtm)) + \ldisti(\ldistm(\dtwo + \gta)),
\end{align*}
which simplifies to $\dtwa \leq \dtwo \cdot \gtm + \dtwo + \gta$ and ultimately
\begin{align*}
\dtwa \leq \dtwo \cdot (1+\gtm) + \gta
\end{align*}
\hfill $\ensuremath{\Box}$
\end{proof}

When $\ubm=0$ or $\uba=0$, the proof remains valid; moreover, the concavity condition 
can be dropped in this case, since $\ldisti$ is no longer applied to a sum. This leads to the following corollary. 
\begin{corollary}
For a monotonically increasing $\ldistm$, setting  $\uba = \ldistm(\dtwo+\gta) - \ldistm(\dtwo)$ and $\ubm=0$ guarantees
$\dtwa \leq \dtwo  + \gta$, and setting $\ubm =  \ldistm( \dtwo\cdot\gtm)  / \ldistm(\dtwo)$ and $\uba=0$ guarantees $\dtwa \leq \dtwo \cdot (1+ \gtm)$.
\end{corollary}

\subsection{Phase 2: Merging segments}
\label{sec:phase2}

The construction of the simplified path discussed before is greedy: splits are introduced one at a time, and the introduction of more split points later on may make earlier split points no longer necessary.  For this reason, in phase 2, a pass is made over segments to see which ones can be merged with their neighbor without violating tolerances. The algorithm for doing this is quite straightforward and can be found in Appendix~\ref{apx:merge}.  The main points worth mentioning are:
\begin{itemize}
\item The order in which segments are considered for merging is short-to-long. The motivation for this is that short segments are considered less desirable from an explanatory point of view: their properties are more likely to be affected by noise.  
\item In principle, a global tolerance criterion can be used at this stage (checking the cost of the full sequence matching rather than one segment).  The algorithm included in Appendix~\ref{apx:merge} uses the local criterion.\footnote{We also provide the global version in our code: 
\url{https://github.com/wannesm/dtaidistance}
}
\end{itemize}


Figure~\ref{fig:merging} illustrates the effect of merging on an example. In Phase 1, the path was consecutively split at E, C, B and D.  In Phase 2, the segments (B, C) and (C, D) are merged into (B, D), which is subsequently merged with (D, E) to form (B, E), all while maintaining adherence to the tolerance criterion. Apparently, path (B, E) captures the most significant relationship (compression) between the two time series. The orange sine wave is well aligned to the blue sine wave, without introducing distracting or insignificant details.



\begin{figure} 
\begin{subfigure}[t]{.45\textwidth}
	\includegraphics[width=\linewidth]{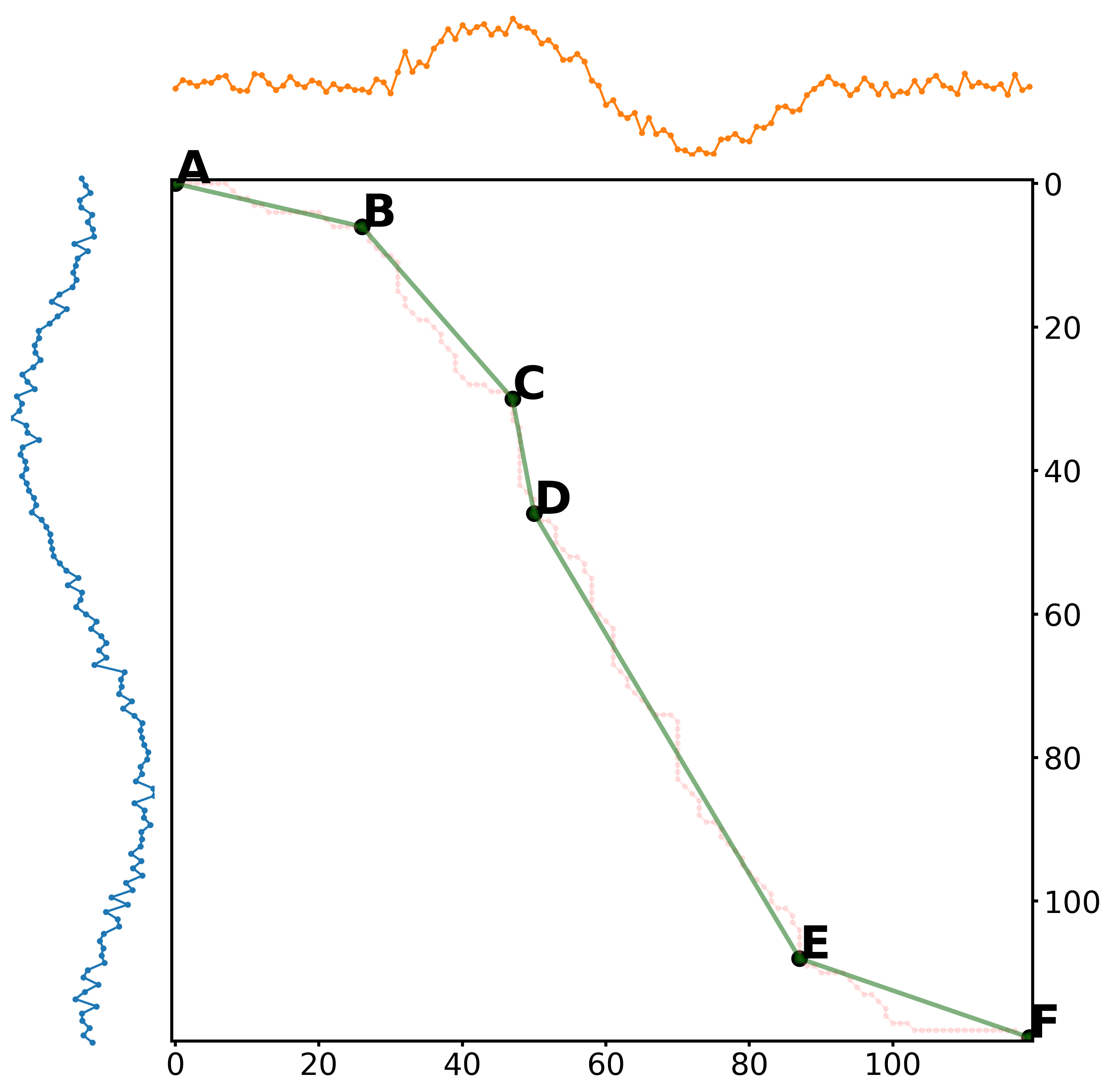}
        \caption{Before merging}
\label{fig:approx_path_sine_nopruning}
\end{subfigure}%
\hfill%
\begin{subfigure}[t]{.45\textwidth}
  	\includegraphics[width=\linewidth]{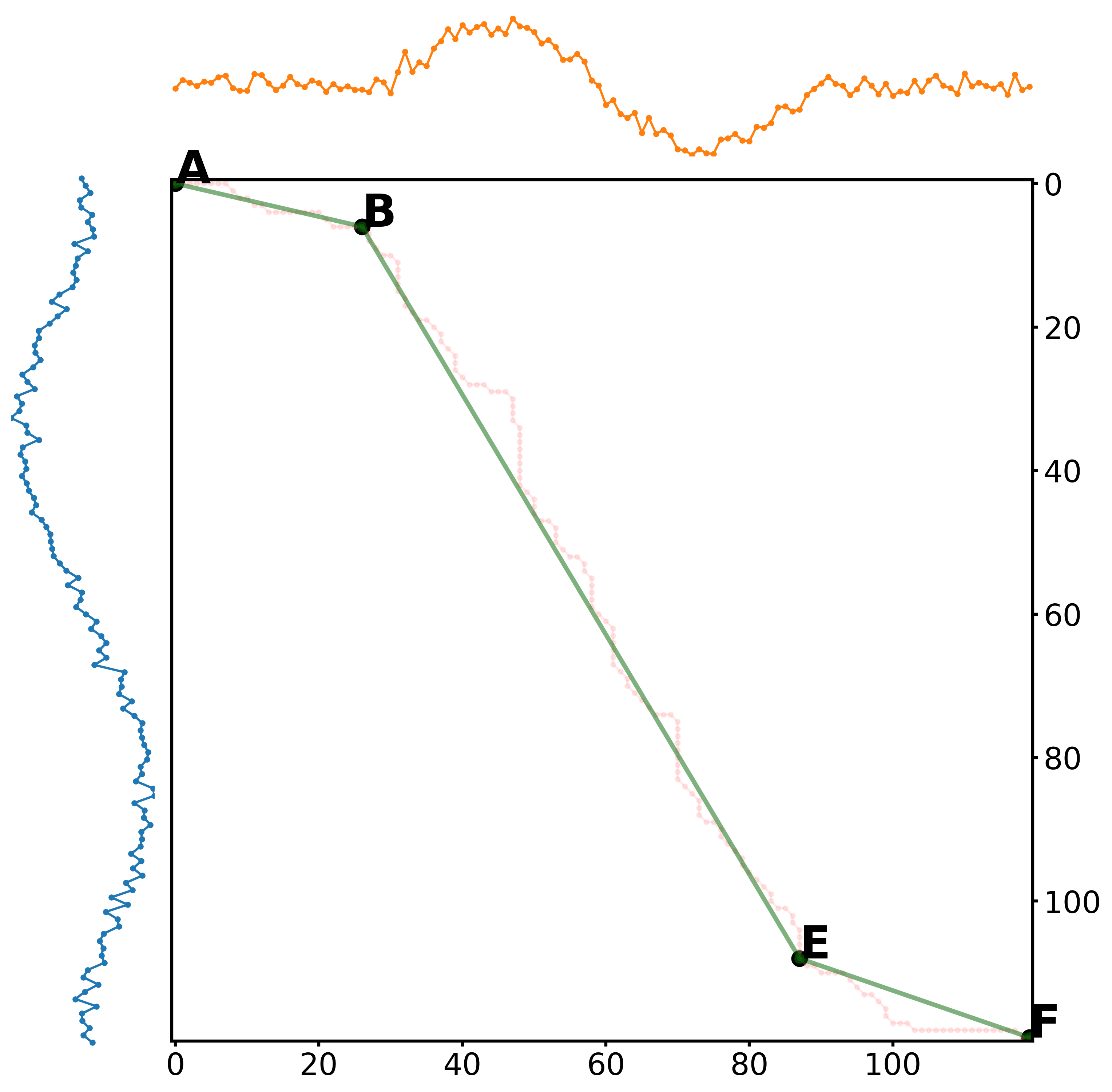}
\caption{After merging} 
\label{fig:approx_path_sine_withpruning}
\end{subfigure}
\caption{Before vs. After the merging phase.}
\label{fig:merging}
\end{figure}





\section{Interpreting and Quantifying the Simplified Path}

The segmentation obtained using the method just described can be used in many ways to characterize the relationship between two time series.  The relationship can be visualized by connecting only the key points (boundaries between segments) pairwise, rather than all points.  The segments are then clearly visible, and show how segments map to very similar segments (modulo shift and uniform compression) in the other series. This was illustrated in Figure~\ref{fig:intro_segmented}.

Besides visualisation, the method makes it possible to numerically characterize individual segment mappings. For instance, to characterize the amount of compression of a segment $\mathbf{s}_1(b:e)$ mapped to  $\mathbf{s}_2(b':e')$, we can define absolute and relative compression (expressed as a log-ratio), relatively denoted $K$ and $\kappa$, as follows:
        \[
        K = (e'-b') - (e-b)
        \]
   
        \[
        \kappa = \log\left(\frac{e'-b' } {e-b }\right)
        \] 
A negative $K$ or $\kappa$ indicates compression (from $\mathbf{s}_1$ to $\mathbf{s}_2$), a positive $K$ or $\kappa$ expansion. In the visualization in Figure~\ref{fig:intro_segmented}, the length of the orange block is $|K|$. $\kappa$ is indirectly visualized in the path visualization in the cost matrix: it is related to the slope of a path segment (with $\kappa=0$ being a slope of $-1$ and $\kappa=\pm \infty$ corresponding to horizontal/vertical path segments).  See Figure~\ref{fig:simplified_path_illustration} for an illustration.

Similarly, the time shift between a segment and the one it maps to can be characterized using a single number.  Depending on the context, one could choose the time  difference between starting points of the segment, the end point, the minimal time difference between any points mapped to each other, etc. These definitions correspond to the horizontal distance from the diagonal of respectively the leftmost, rightmost, and closest-to-the-diagonal point of the segment.  The shift according to the third definition can be calculated efficiently as 
\[
 \sigma =
 \begin{cases}
      0, & \text{if } (b'-b)  \cdot (e'-e) < 0  \\
     b'-b & \text{if } |b'-b| \leq |e'-e| \\
     e'-e & \text{otherwise}
 \end{cases}
\]

It can also be useful to visualize the vertical (``amplitude'') difference between the segments after warping.  In Figure~\ref{fig:intro_segmented}, this difference is visualized as orange shading. This mapping is shown for individual indexes. Let $\Paf$ denote the full rasterized path obtained from the subsequence mappings. As the mapping is not one-to-one, a single index of $\mathbf{s}_1$ may map to multiple indexes of $\mathbf{s}_2$.  The shading around an index $i$ of $\mathbf{s}_1$ therefore goes from $v_i$ to the extremal values of indices matched with $i$; that is, to $v_i+\alpha(i)$ and $v_i-\beta(i)$ with
\begin{align*}
       \alpha(i)  \gets \max \{ |\tst(j) - \tsf(i)|  \mid   (i, j) \in \Paf \land \tst(j) \geq \tsf(i), \quad 0\}\\
    \beta(i)  \gets \max \{|\tst(j) - \tsf(i)| \mid (i, j) \in \Paf \land \tst(j) \leq \tsf(i), \quad 0\}
\end{align*}

\begin{figure}[t]
\centering
\includegraphics[width=0.5\linewidth]{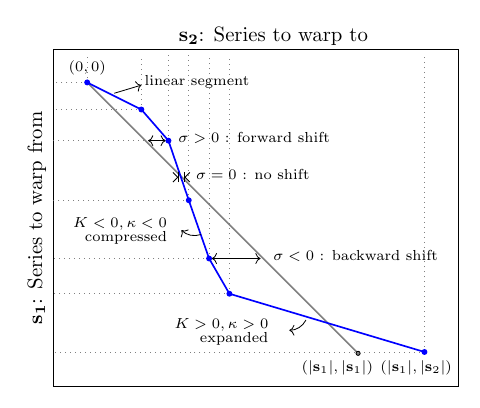}
\caption{The simplified path (in blue) and how the segments can be geometrically interpreted in terms of time-shift (forward shift, backward shift or no shift) and compression (compressed or expanded).
}
\label{fig:simplified_path_illustration}
\end{figure}

\section{Showcases}
\label{sec:examples}

To illustrate the use of Dynamic Subsequence Warping (DSW), we present three examples where our approach is applied to pairs of time series taken from the widely used UCR collection of time series datasets~\cite{keogh14ucr}. Similar time series are found by selecting time series with the same label.

\paragraph{Example 1:}
We show how the comparison between two time series with similar heart beats from the ECGFiveDays dataset is easier to interpret with the DWS visualization. The DSW visualization in Figure~\ref{fig:sc_ecg_dsw} effectively captures meaningful subsequence mappings.\footnote{Note that finding meaningful mappings to show (dis)similarity is not the same as finding semantically meaningful segments.  E.g., in ECG signal analysis, experts typically segment the signal into intervals corresponding to the distinct phases of cardiac activity. DSW does not have this background knowledge. However, DSW could be constrained to yield subsequences compatible with semantic segments by starting the path splitting at semantic segment boundaries (and preventing merging around those boundaries), thus combining its strengths with those of semantic segmentation.} The peaks are mapped to each other and the difference between the two series is clearly visualized. First, the peak around index 50, the shading shows that the peak has a different height but is nearly identical otherwise. Second, the peak around index 80 is shown to be both higher and wider, visualized by the amplitude shade and the expansion of the segment.
The standard visualization with point-to-point alignments, shown in Figure~\ref{fig:sc_ecg_dtw}, is more difficult to interpret, especially because of the extreme warping towards the end of the series (i.e., $1$-to-$n$ or $n$-to-1 mapping with $n$ large). This occurs because DTW is sensitive to noise, even if the benefits of such extreme warping are minimal. The DSW visualization is robust to such noise as it identifies that these benefits are small and can simplify the representation accordingly. 
When inspecting the path in Figure~\ref{fig:sc_ecg_both}, one can observe the extreme warping in the lower right corner. The simplification would have been rejected by the RDP algorithm since it relies solely on spatial information and the original path is too far from the simplification. DSW, on the other hand, realizes the path segment with the extreme warping and the simplified path have a similar cost even though they are spatially distant of each other.

\begin{figure} [t]
    \begin{subfigure}[t]{0.32\textwidth}
            \includegraphics[width=\linewidth]{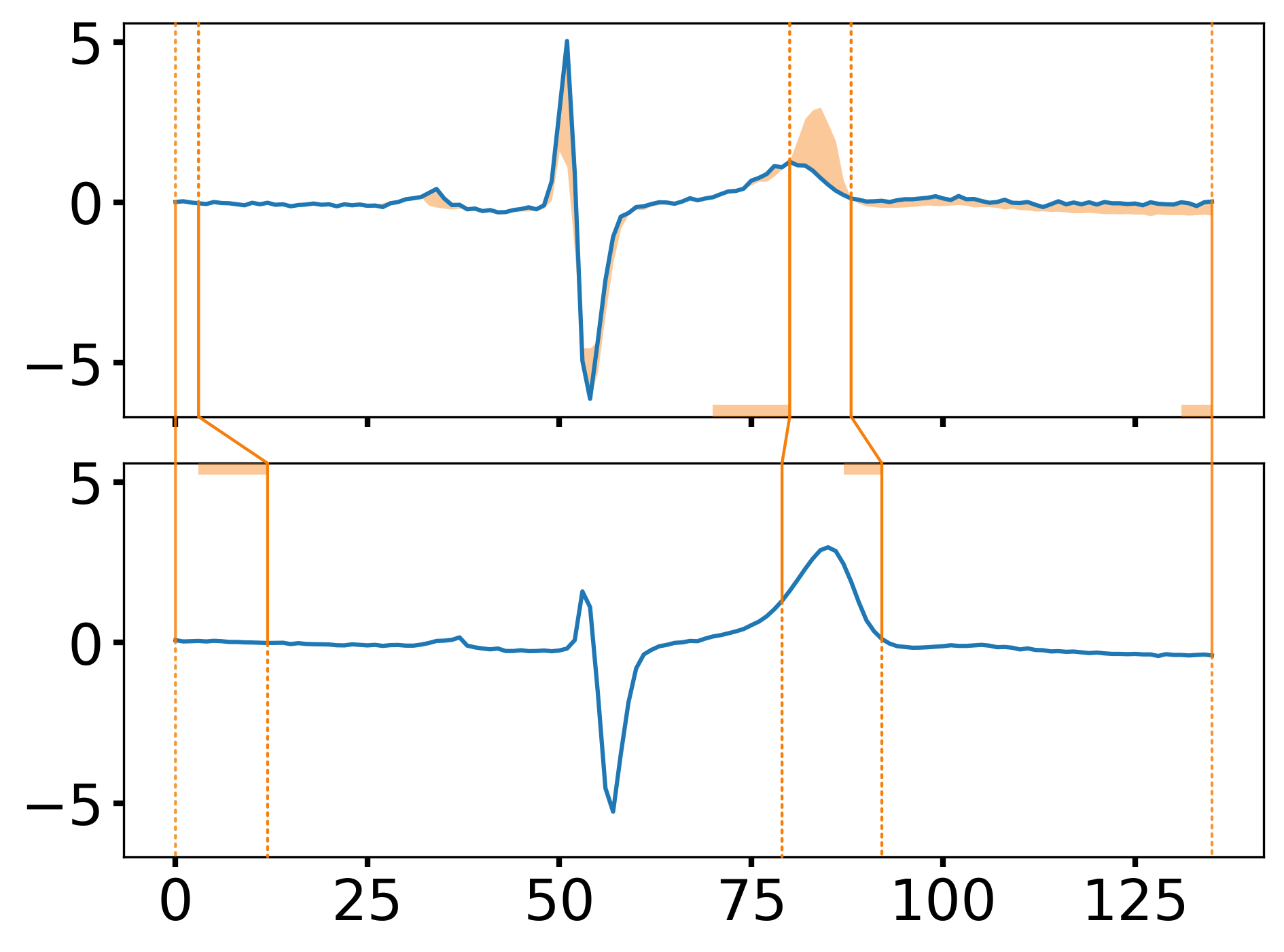}
                \caption{DSW visualization}
            \label{fig:sc_ecg_dsw}
    \end{subfigure}
    \begin{subfigure}[t]{0.32\textwidth}
            \includegraphics[width=\linewidth]{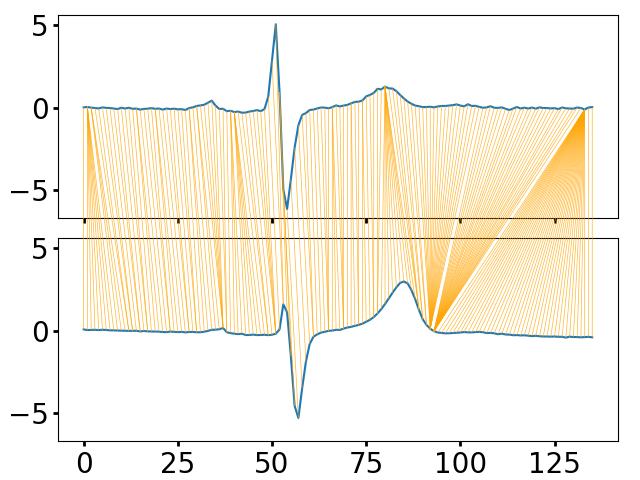} 
                \caption{Standard Visualization 
                }
            \label{fig:sc_ecg_dtw}
    \end{subfigure}
    \begin{subfigure}[t]{0.32\textwidth}
            \includegraphics[width=\linewidth]{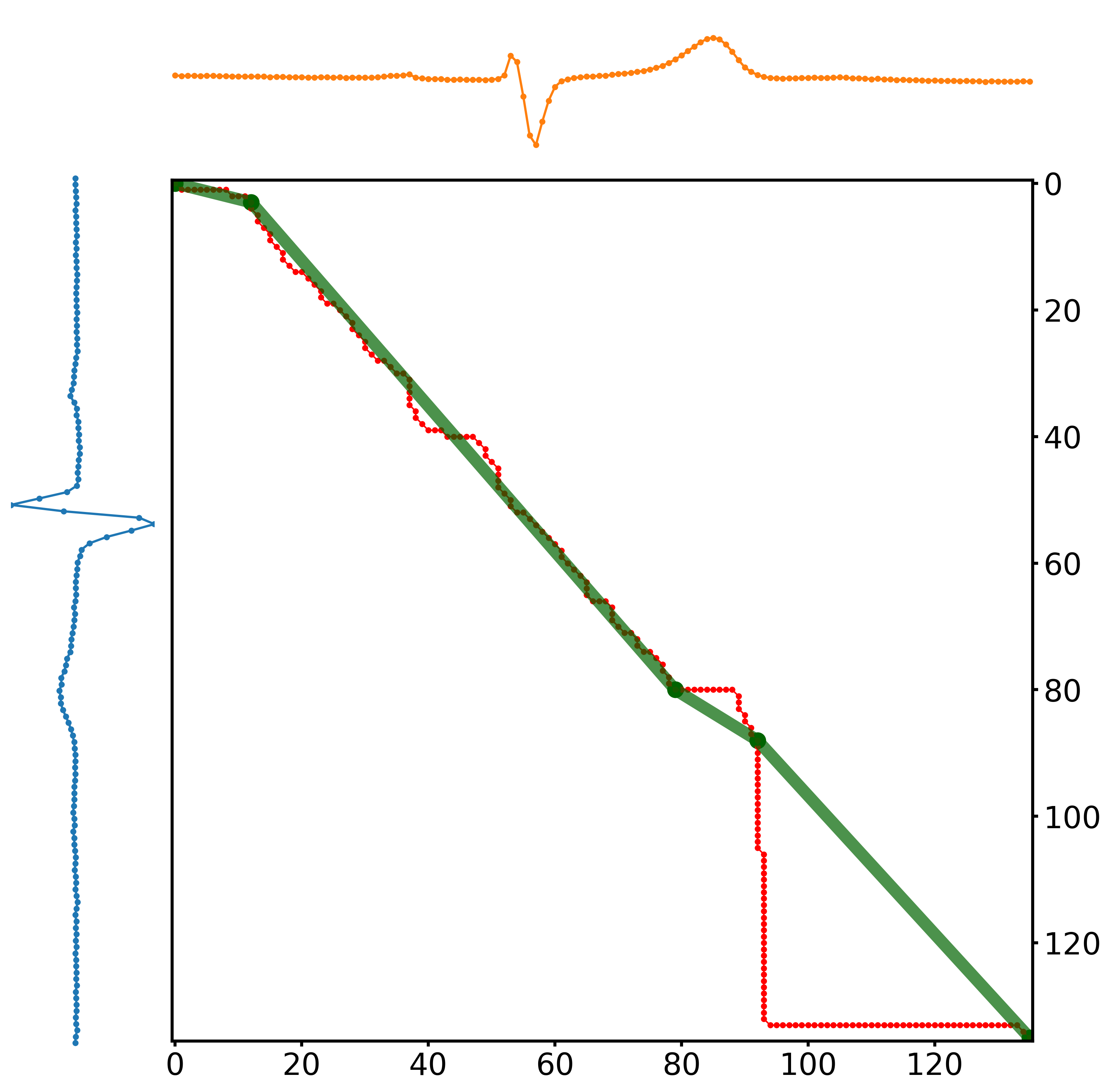} 
                \caption{Original (red) vs. Simplified (green) Path}
            \label{fig:sc_ecg_both}
    \end{subfigure}
     \caption{Comparison of (a) the DSW visualization vs. (b) visualization showing all matched values for two time series where (c) the path overfits to noise.}
    \label{fig:ecg}
\end{figure}
    
\paragraph{Example 2:}
We show how the combined relative and absolute tolerance criterion is useful in practice. Figure~\ref{fig:sc_umd_combi} depicts the DSW visualization for two time series from the UMD dataset with the combined tolerance criterion. After switching to the relative tolerance criterion (using the same value for $\ubm$) in Figure~\ref{fig:sc_umd_rel}, we can observe that the compression over the valley is still identified. But the beginning of the series is now split in multiple segments because they have a very low cost and DSW is overfitting the noise. Adding the absolute tolerance allows the simplification a bit more margin to simplify segments with small deviations of the values.


\begin{figure}[t]
\centering
 \begin{subfigure}{.40\textwidth}
 \centering
	\includegraphics[width=\linewidth]{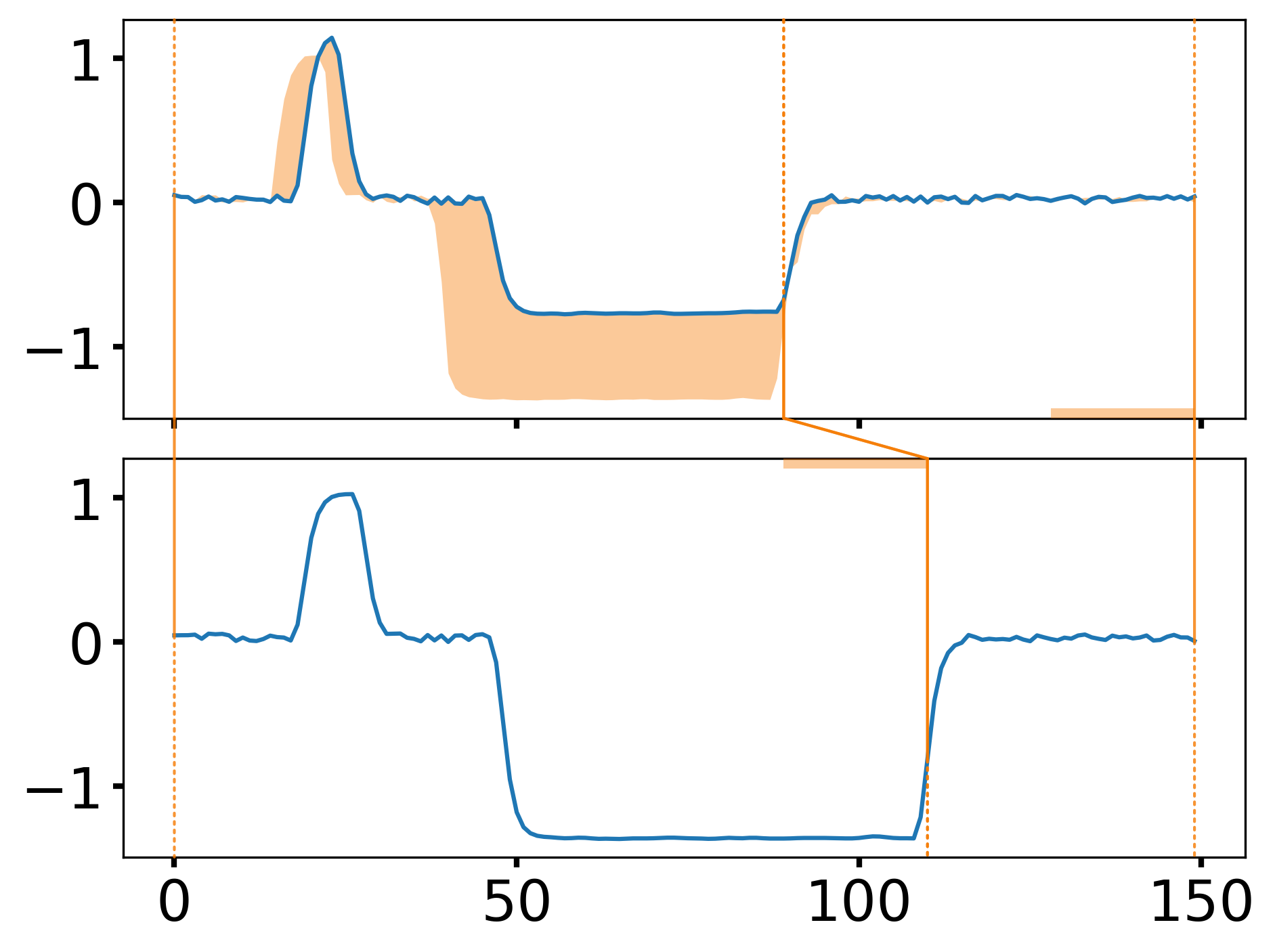}
    \caption{}
    \label{fig:sc_umd_combi}
\end{subfigure}
\hspace{1em}
\begin{subfigure}{.40\textwidth}
\centering
	\includegraphics[width=\linewidth]{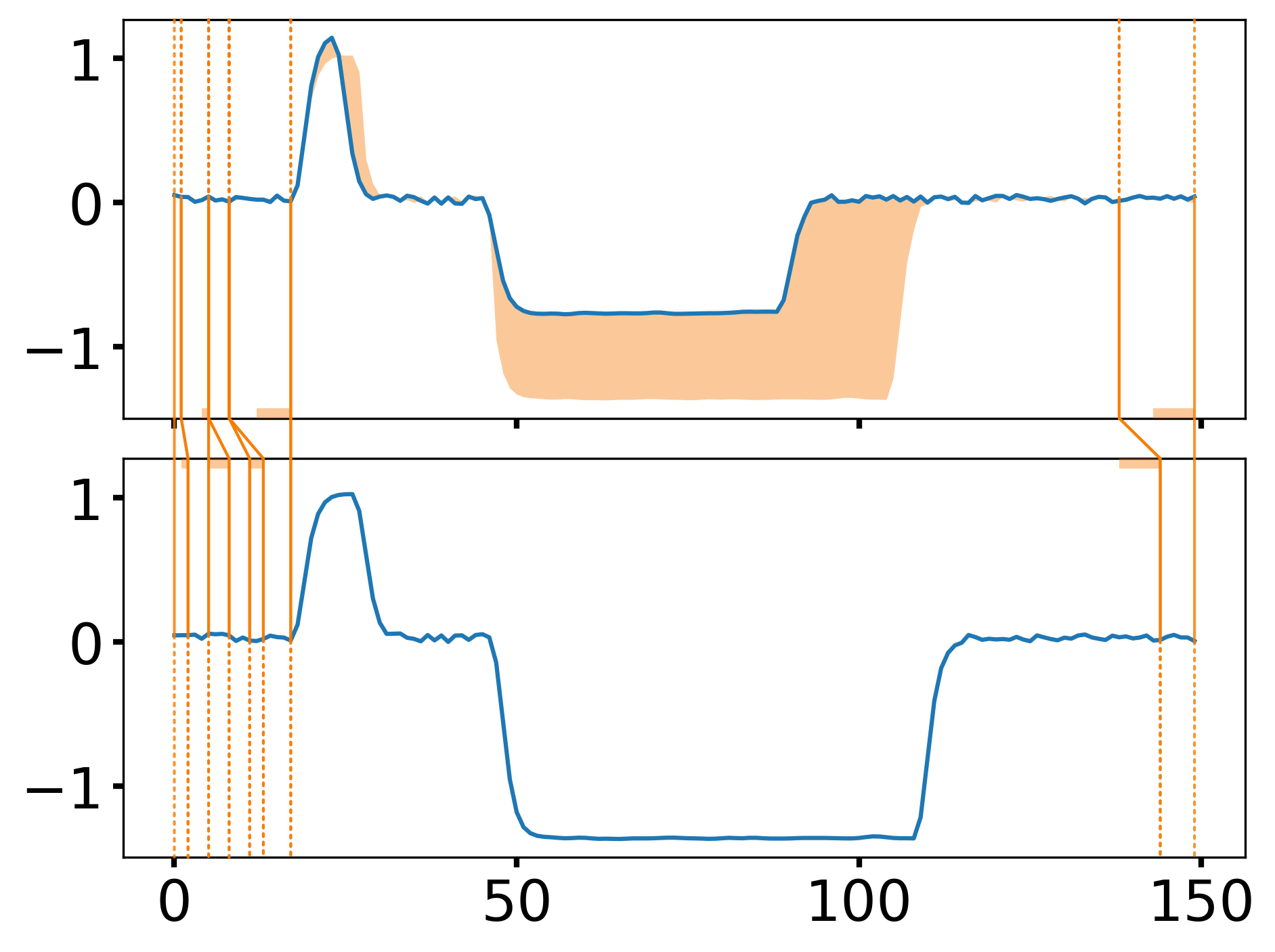}
    \caption{}
    \label{fig:sc_umd_rel}
\end{subfigure}
\caption{The DSW visualization using (a) the combined tolerance criterion vs. (b) the relative tolerance criterion}
\label{fig:sc_umd}
\end{figure}

\paragraph{Example 3:}
In these examples, to align two similar series, a large shift and compression is needed. This complicates the point-to-point visualization because the connecting lines get close to each other, or even cross. The examples in Figure~\ref{fig:sc_cbf_and_two_pattern}, taken from the CBF and TwoPatterns datasets, illustrate this effect and show how the DSW visualisation provides a more clear and intuitive subsequence-to-subsequence matching.

\begin{figure}[t]
\centering
\begin{subfigure}[t]{.32\textwidth}
    \centering
    \includegraphics[width=\linewidth]{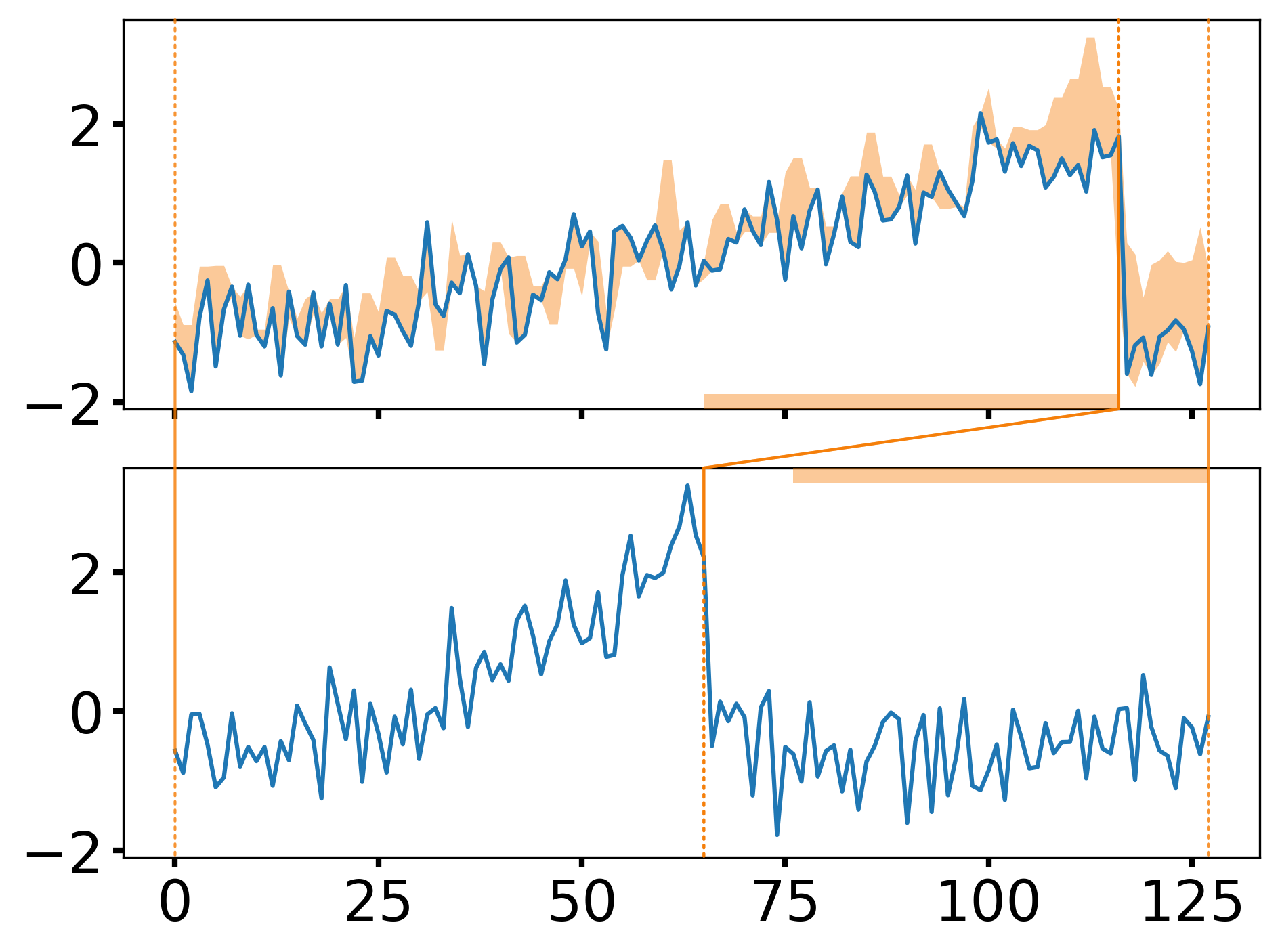}
    \caption{}
    \end{subfigure}
\hspace{1em}
\begin{subfigure}[t]{.32\textwidth}
    \centering
    \includegraphics[width=\linewidth]{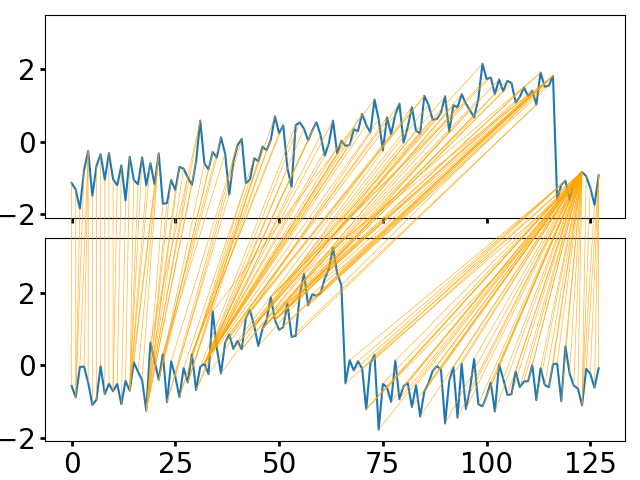}
        \caption{}
\end{subfigure}

\begin{subfigure}[t]{.32\textwidth}
    \centering
    \includegraphics[width=\linewidth]{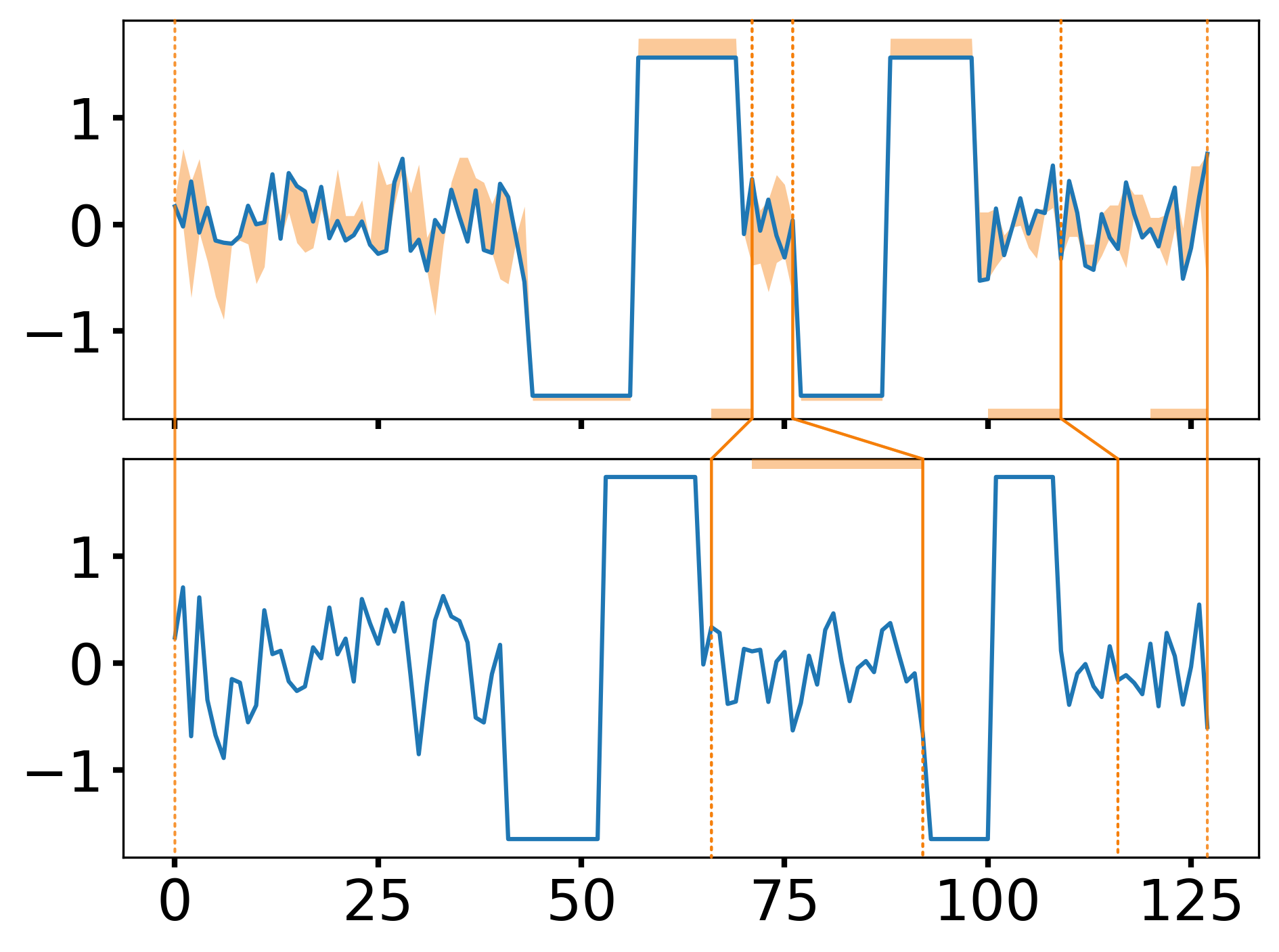}
        \caption{}
        \label{fig:sc_two_pattern}
\end{subfigure}
\hspace{1em}
\begin{subfigure}[t]{.32\textwidth}
        \centering
        \includegraphics[width=\linewidth]{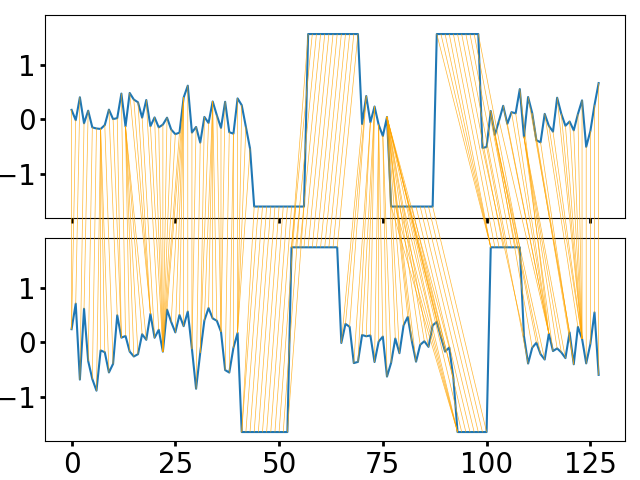}
            \caption{}
\end{subfigure}
\caption{The DSW Visualization vs. Standard Visualization on the CBF dataset (top) and the TwoPatterns dataset (bottom).}
\label{fig:sc_cbf_and_two_pattern}
\end{figure}

More examples, including a paused time series, different lengths, a pair of dissimilar time series, etc. can be found in Appendix~\ref{apx:showcases}.



\section{Related Work}
\label{sec:relwork}

\paragraph{Visualizing the warping between time series.}
Visualizations that show how elastic distances capture (dis)similarities between time series are mainly found in application areas. These methods visualize how individual values match to each other. They do not directly identify subsequences and allow an easy interpretation about how they match to each other. As a consequence, these methods only consider shifts since quantifying compression of multiple points requires identifying subsequences.
%
For instance, the LibRosa framework uses DTW for music synchronization \cite{mcfee2015librosa}. To visualize the warping between two pieces of music, they use a visualization similar to Figure~\ref{fig:intro}.a that connects the values that match in the warping path. Different is that they only show every $n$-th matching in the warping path to simplify the visualization as the many lines can be overwhelming ($n$ is a user-defined parameter). But this hides changes in behavior that occur in between matches that are selected to be shown.
Another visualization is available in the MocapViz toolbox that is designed to analyze motion sequences \cite{Budikova2022ICSSA}. To visualize the differences between two aligned motion sequences (e.g., pose to pose), they make use of a visualization that connects the indices that are matched in the warping path. By connecting the indices instead of the values, the drawn lines cannot cross each other and the result is more clean. Additionally, they use a misalignment function to color the lines that connect matching indices to indicate speedups and slowdowns (see Figure~\ref{fig:intro}.b). The coloring, however, only shows very local effects as it uses the time difference between two consecutive points. The interpretation of how subsequences match to each other is left to visual inspection and is difficult to do.

\paragraph{Constraining the warping path.}
Another line of research in the field of elastic distances concerns obtaining more diagonal warping paths. The goal of these methods is rather to improve the quality of the distance than to illustrate and explain the warping path. This is typically achieved by changing the DTW algorithm itself. Already in the original papers introducing DTW, it was proposed that the warping path could be limited to only take certain steps in the cost matrix and in effect prefer steps that follow the direction of the diagonal (e.g., by not allowing repeated horizontal or vertical steps) \cite{Sakoe1978TASS}. Alternatively, this preference to obtain warping paths that follow the direction of the diagonal can be achieved by adding a penalty to the cost when the path deviates from the diagonal step (also referred to as amerced DTW or ADTW) \cite{Clifford2009AC,Herrmann2023DMKD}. However, both strategies express only a preference for diagonal paths, thus paths with $\kappa=0$ where there is no compression. Also, there is no guarantee that (parts of) the resulting path are actually linear. An advantage of these techniques is that they reduce the overfitting on noise like in figure~\ref{fig:ecg}. But with the disadvantage that a large compression is not allowed anymore. Also when it might be useful (e.g., comparing two monitored processes of which one was temporarily paused, see Appendix~\ref{apx:showcases} for an example).
While our method is demonstrated using the optimal path returned by DTW, it works for any warping path and is thus compatible with any framework that determines a warping path, including the DTW variants that constrain the path.


\paragraph{Polyline simplifications.}
While the approach presented in this work is based on methods to simplify a polyline (i.e., the RDP algorithm), such methods have a very different goal. Their goal is to keep as close as possible to the original shape \cite{Douglas1973Carto,Ramer1972CGIP}. The simplification criterion in that case requires that the deviation, defined as the spatial distance between the farthest point on the original curve and the line, remains below a specified threshold. However, in our context, this definition of the deviation is no longer meaningful (see Section~\ref{sec:examples}). Simplifications of the warping path should only be allowed if the two series behave similarly as defined by the accumulated cost over the path. As a consequence, the method presented in this work will both allow simplifications that RDP will not allow, and will not allow simplifications that RDP will allow.


\section{Conclusion}
We presented the Dynamic Subsequence Warping algorithm and associated visualization that allows an interpretable and quantifiable comparison between two time series.
Our method works by identifying key alignments (of subsequences) within the optimal warping path, producing a piecewise linear path. It reduces excessive or unnecessary warping while maintaining a small total distance. Since each segment can be quantified, users can easily assess how much one time series (subsequence) leads, lags, expands, compresses, or differs in amplitude compared to the other.

\bibliographystyle{splncs04}
\bibliography{references}

\newpage
\appendix

\section{Bresenham's Line Rasterization Algorithm}
\label{apx:bresenham}

For a uniform subsequence mapping, a warping path is needed that represents the linear interpolation between two points in the cost matrix. We use Bresenham's line rasterization algorithm for this. The paths produced with this algorithm satisfy the three conditions for a warping path.



\begin{algorithm}[H]
\caption{Bresenham's line rasterization algorithm to compute the accumulated cost over a linear path between two points.}
\label{alg:bresenham}
\begin{algorithmic}[1]
\Procedure{LinearPathCost}{$\Po, \Cm, b, e$}
\ARGUMENTS
\State{$\Po$: Optimal warping path}
\State{$\Cm$: Cost matrix}
\State{$b, e$: Starting and ending index of the path}
\ENDARGUMENTS
\RETURNS
\State{Accumulated cost}
\ENDRETURNS
\State {$q_0 \gets \Po(b)$}
\State {$q_1 \gets \Po(e)$}
\State $d_f,d_t \gets q_1(0)-q_0(0),q_0(1)-q_1(1)$\Comment{Differences}
\State $e \gets d_f + d_t$\Comment{Error}
\State $q_i \gets q_0$\Comment{Start point}
\State $c_a \gets 0$\Comment{Accumulated cost for linearized path}
\While{$q_i \neq q_1$}
\If{$q_i \neq q_0$}
\State $c_a \gets c_a + \Cm(q_i)$
\EndIf
\If{$2e \geq d_t$}
\State $e \gets e + d_t$
\State $q_i \gets q_i + (1,0)$
\EndIf
\If{$2e \leq d_f$}
\State $e \gets e + d_f$
\State $q_i \gets q_i + (0,1)$
\EndIf
\EndWhile
\State $c_a \gets c_a + \Cm(q_1)$
\State \textbf{return} $c_a$
\EndProcedure
\end{algorithmic}
\end{algorithm}

\newpage
\section{Algorithm for Merging Uniform Subsequence Mappings (Phase 2)}
\label{apx:merge}

In section~\ref{sec:phase2}, we have provided all relevant steps to merge segments in phase 2 of the Dynamic Subsequence Warping algorithm. Since that is a high level description, we provide all details in Algorithm~\ref{alg:merge}.

\begin{algorithm}[H]
\caption{Phase 2: Merging segments.}
\label{alg:merge}
\begin{algorithmic}[1]
\Procedure{MergeSegmentedPath}{$\Po,\Cm, \mathbf{R}, \uba, \ubm$}
\ARGUMENTS{}
\State{$\Po$: Optimal warping path}
\State{$\Cm$: Cost matrix}
\State {$\mathbf{R}$: Indices of $\Po$ that make up the simplified path}
\State {$\uba$, $\ubm$: Absolute and relative tolerance criterion}
\ENDARGUMENTS
\State $\mathbf{Q} \gets \textsc{PriorityQueue}(\{\})$\Comment{Prioritize shortest segments}
\For{$i \in (2:|\Ps|-1)$}
\State $\mathbf{Q} \gets \mathbf{Q}  + \left\{(\min\left[\Ps(i)-\Ps({i-1}), \Ps({i+1})-\Ps(i)\right], \Ps({i-1}), \Ps(i), \Ps({i+1}))\right\}$ 
\EndFor
\State $\mathbf{R} \gets \textsc{SortedList}(\mathbf{R})$
\While{$|\mathbf{Q} | \neq 0$}
\State{Take the first adjacent segment pair $(b,m, e)$ from $\mathbf{Q}$.}
\State $l \gets e - b$
\State $c^\prime \gets \Call{LinearPathCost}{\Po, \Cm, b, e}$
\State $c \gets \Call{OptimalPathCost}{\Po, \Cm, b, e}$
\If{$c^\prime \leq \max(c + \frac{l}{L} \cdot \uba, c \cdot (1 + \ubm))$}
\State {Remove $m$ from $\mathbf{R}.$}\Comment{Merge segments $(b, m)$ and $(m, e)$}
\If{$b > 1$}
\State $p \gets$ largest index in $\mathbf{R}$ that is less than $b$  
\State{Add $(\min(b-p, e-b), p, b, e)$ to $\mathbf{Q}$.}
\EndIf
\If{$e < |\Ps|$}
\State $n \gets$ smallest index in $\mathbf{R}$ that is greater than $e$  
\State{Add $(\min(e-b, n-b),  b, e, n)$ to $\mathbf{Q}$.}
\EndIf
\EndIf
\EndWhile\label{alg:merge-while}
\State $\mathbf{P}^{s} \gets \{\Po (i) \mid i \in \mathbf{R}\}$
\State $\textbf{return } \mathbf{P}^{s}$
\EndProcedure
\end{algorithmic}
\end{algorithm}

\newpage
\section{Additional examples} \label{apx:showcases}

In this appendix, we collected a few cases that are not presented in the main text but can be useful to deepen the understanding of what the visualization represents.

\paragraph{Example: Two monitored processes where one is paused.}
Extreme warping is when the warping path contains a $1$-to-$n$ or $n$-to-$1$ mapping in the warping path with $n$ large. In some cases, this unexpected behavior is due to overfitting to noise like in Figure~\ref{fig:ecg} and we want to simplify such a path. This is also the setting where techniques like constraining the warping path (see Section~\ref{sec:relwork}) can be useful to avoid extreme warping and obtain a better, more robust, path and distance. In other cases, this is the expected behavior and we want to capture this extreme warping. For example, when two processes are monitored and one is paused for a certain amount of time. In Figure~\ref {fig:apx_ecg_5_pause} we have reused the time series from Figure~\ref {fig:sc_two_pattern} and introduced a pause in the second time series causing the value to stay constant during the duration of the pause. We can observe that this behavior is captured by the segmentation as expected.

\begin{figure}[h]
\centering
\begin{subfigure}[t]{.48\textwidth}
    \includegraphics[width=\linewidth]{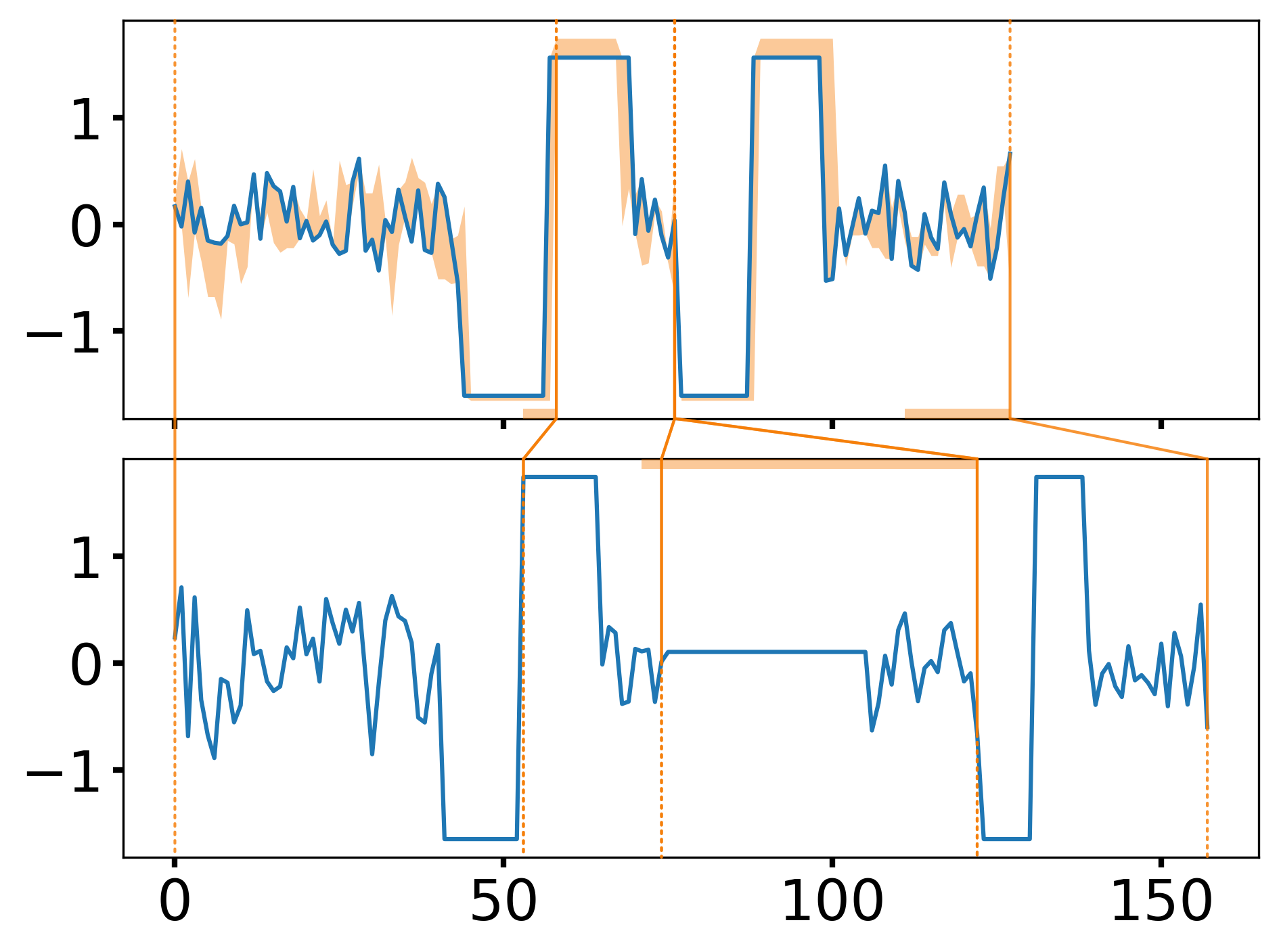}
    \caption{}
\end{subfigure}
\hfill
\begin{subfigure}[t]{.40\textwidth}
    \includegraphics[width=\linewidth]{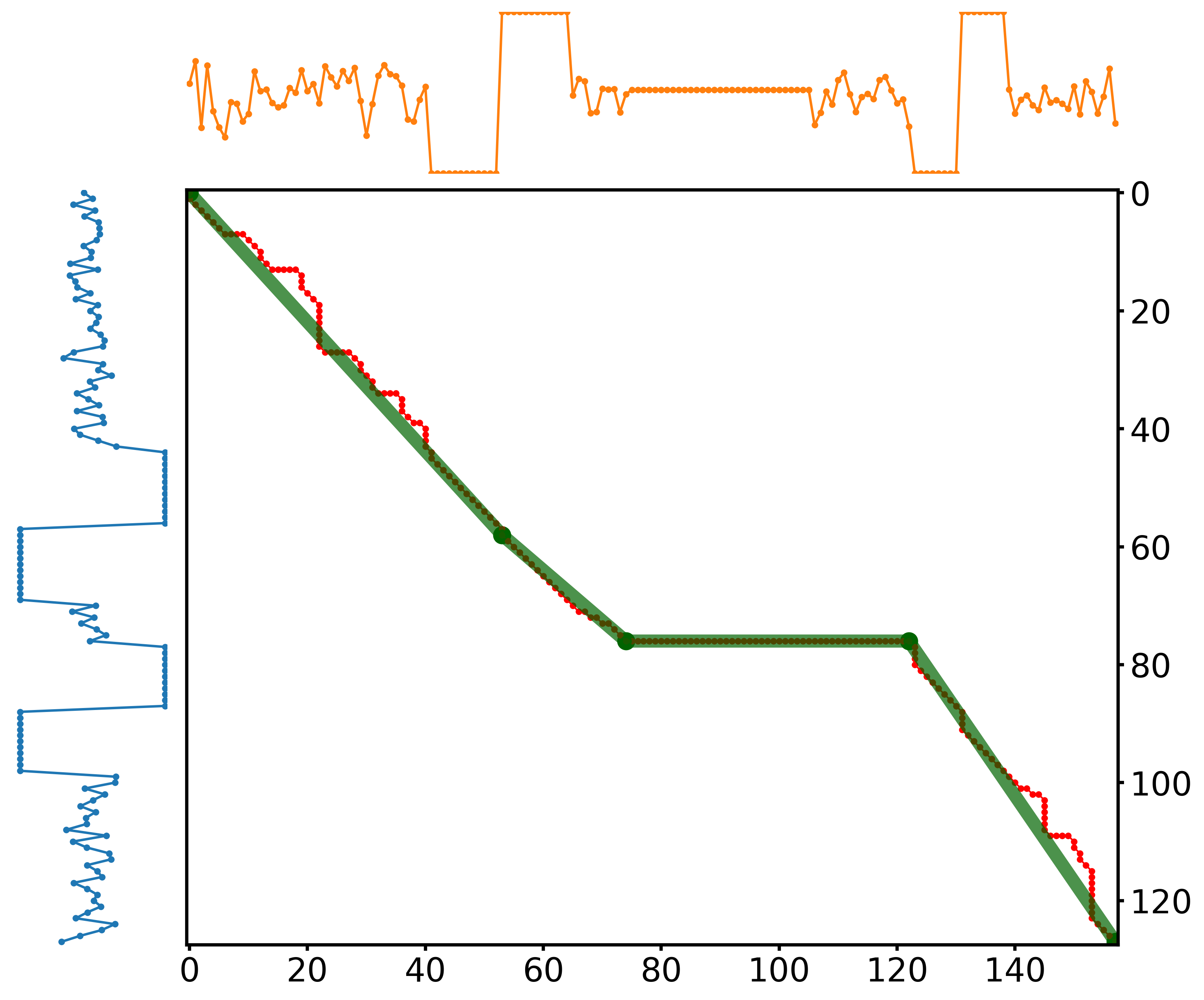}
    \caption{}
\end{subfigure}
\caption{Two time series from the TwoPatterns dataset (UCR collection) with a pause introduced in the bottom time series at index 75 (left) leading to full compression in the segment around index 75 (right)}
\label{fig:apx_ecg_5_pause}  
\end{figure}

\paragraph{Example: DSW visualizations of similar vs. dissimilar time series.}
Below, we illustrate how DSW visualizations can aid in understanding why a given pair of time series are dissimilar or have different labels. In Figure~\ref{fig:apx_arrowhead_similar}, which shows the DSW visualization for a pair of similar time series, there is only a noticeable amplitude difference between the matched subsequences in the middle. This one segment explains why the distance between these two time series is larger than other more identical time series.
In contrast, Figure~\ref{fig:apx_cbf_dissimilar}, which shows the DSW visualization for a pair of dissimilar time series (i.e., different labels), reveals a clear amplitude difference across almost all the matched subsequences. 
Even though the matching discovered similar behavior. For example, the second segment exhibits a decrease of values over the segment. But a sudden and steep decrease in the first time series and a gradual one in the second time series. The amplitude shading in that segment also shows that the second time series starts at a higher value. This also is why the segment on the first time series is still rather wide, its maximal value is somewhat in the middle of the values of the decreasing part in the second time series. All of this information provides supplementary insight into why the time series are dissimilar in this case.

The pair of similar time series (a) is from the ArrowHead dataset of the UCR collection~\cite{keogh14ucr}, while the dissimilar pair is from the CBF dataset of the same collection. In each dataset, we consider time series with the same label as similar and those with different labels as dissimilar.


\begin{figure}
\centering
\begin{subfigure}[t]{.48\textwidth}
    \includegraphics[width=\linewidth]{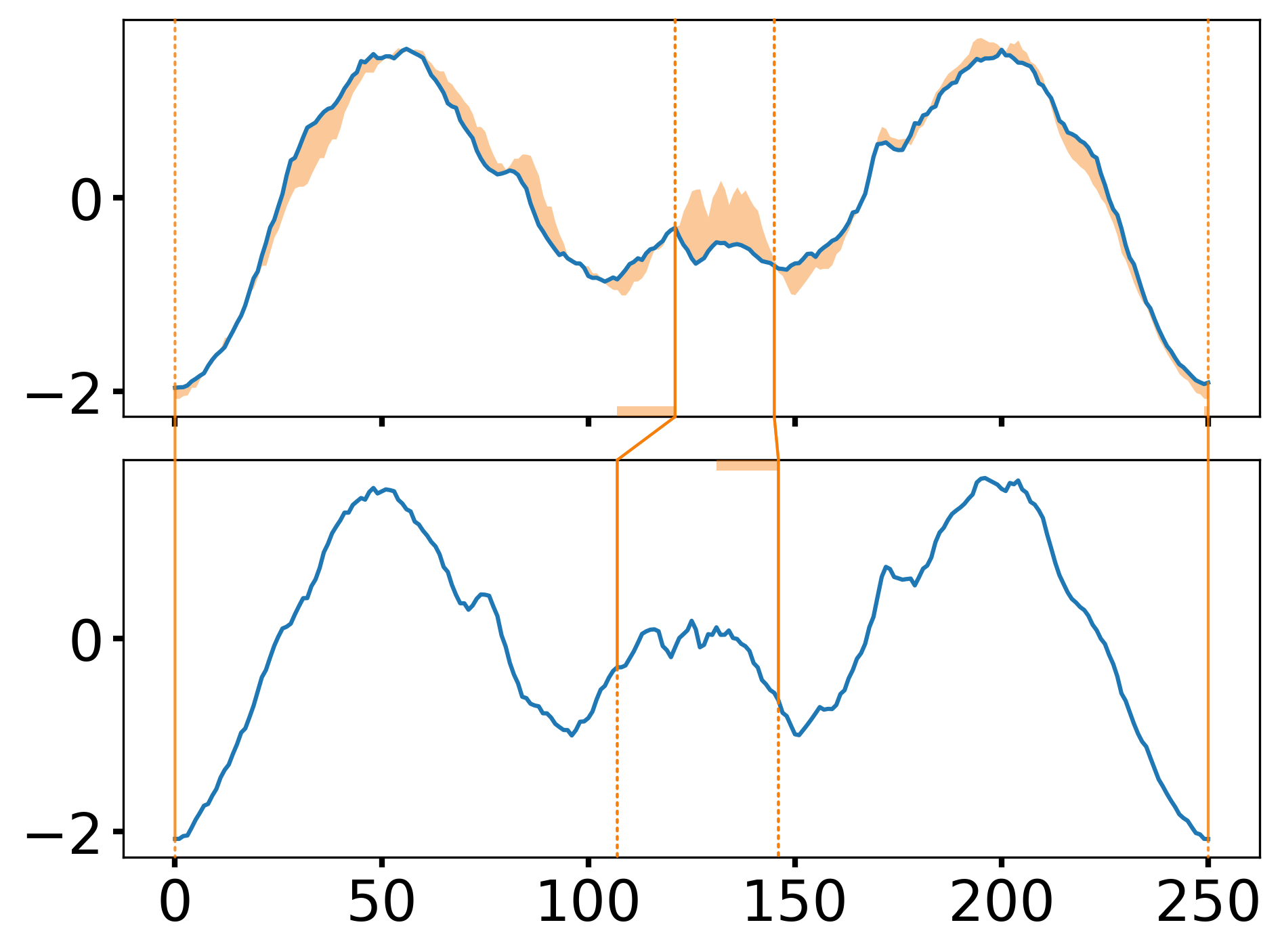}
    \caption{}
    \label{fig:apx_arrowhead_similar}
\end{subfigure}
\hfill
\begin{subfigure}[t]{.48\textwidth}
    \includegraphics[width=\linewidth]{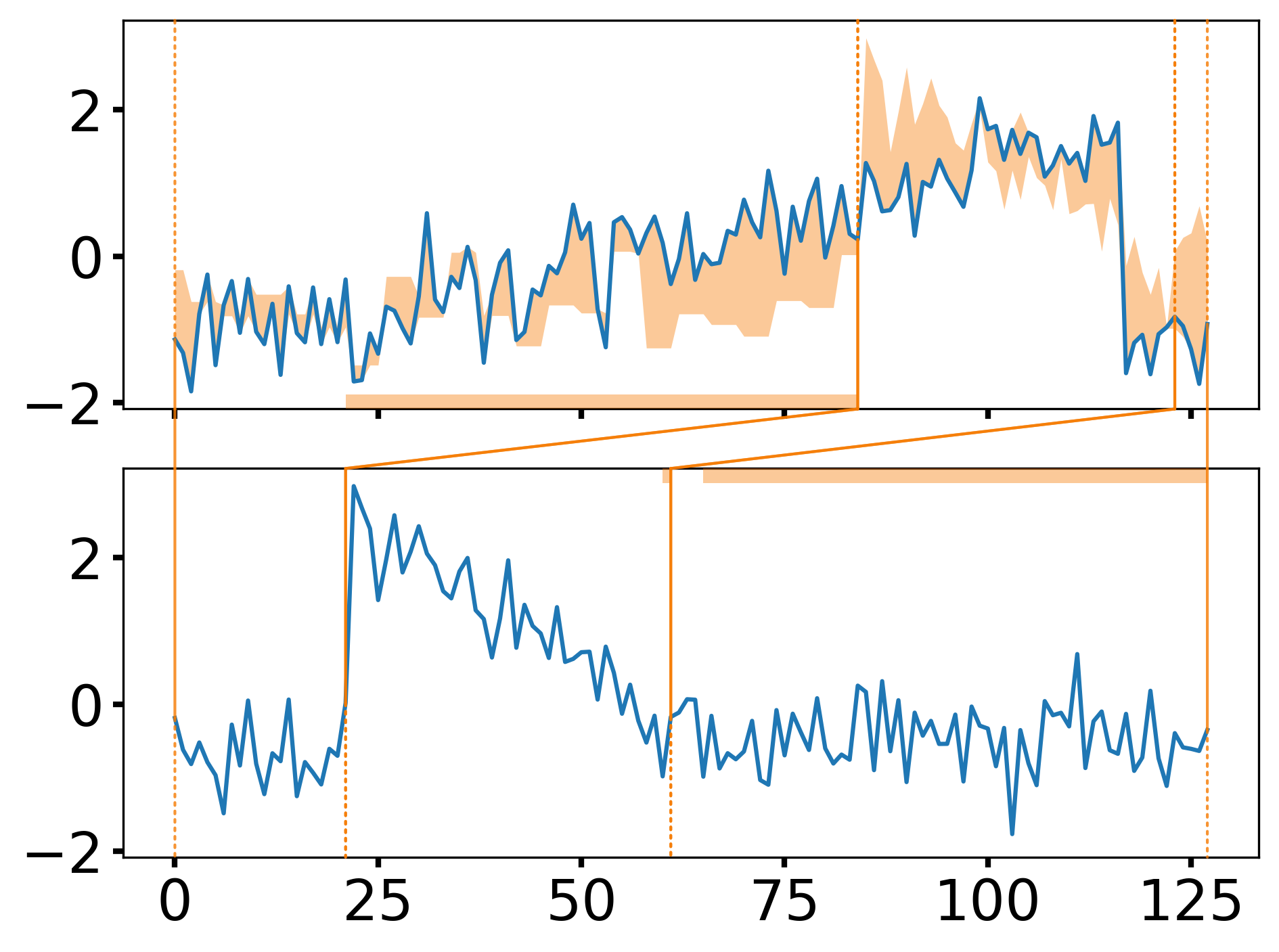}
    \caption{}
    \label{fig:apx_cbf_dissimilar}
\end{subfigure}
\caption{DSW visualization for a time series pair with the same label (a) vs. a time series pair with different labels (b).}
\end{figure}

\end{document}